%% file: EntropicLatent.tex
\documentclass[11pt]{article}
\usepackage[margin=1in]{geometry}

\usepackage[utf8]{inputenc} % allow utf-8 input
\usepackage[T1]{fontenc}    % use 8-bit T1 fonts
\usepackage{hyperref}       % hyperlinks
\usepackage{url}            % simple URL typesetting
\usepackage{booktabs}       % professional-quality tables
\usepackage{amsfonts}       % blackboard math symbols
\usepackage{nicefrac}       % compact symbols for 1/2, etc.
\usepackage{microtype}      % microtypography

\usepackage{wrapfig}

\usepackage{subcaption}
\usepackage{tikz}
\usetikzlibrary{shapes,arrows}
\usepackage{amsfonts}
\usepackage{mathtools}

\usepackage{algorithm}
\usepackage[noend]{algorithmic}
\usepackage{amsthm}
\usepackage{amssymb}
\usepackage{bbm}

\DeclareMathOperator*{\argmax}{arg\,max}

\newtheorem{definition}{Definition}
\newtheorem{lemma}{Lemma}
\newtheorem{corollary}{Corollary}

\newtheorem{theorem}{Theorem}
\newtheorem{assumption}{Assumption}
\newtheorem{conjecture}{Conjecture}

\newcommand{\mat}[1]{\mathbf{#1}}

\newcommand{\del}[1]{{\partial #1}}
\newcommand{\deltwo}[1]{{\partial^2 #1}}
\newcommand{\indep }{\perp\!\!\!\perp}
\newcommand{\ci}[3]{#1 \indep #2 \left\vert #3\right. }
\newcommand{\nci}[3]{#1 \not\!\perp\!\!\!\perp #2 \left\vert #3\right. }

\newcommand{\p}[1]{\mathbb{P}\left[#1\right]}
\newcommand{\pp}[1]{\mathbb{P}\left(#1\right)}
\newcommand{\renyi}{R\'{e}nyi }
\newcommand{\rce}{\text{R\'enyi}}

\usepackage{authblk} 

\begin{document} 
	\date{\today}
	\title{Applications of Common Entropy for Causal Inference}
	\author[1,a ]{Murat Kocaoglu}
	\author[2,b]{Sanjay Shakkottai}
	\author[2,c]{Alexandros G. Dimakis}
	\author[2,d]{Constantine Caramanis}
	\author[2,e]{Sriram Vishwanath}
	
	\affil[1]{\small MIT-IBM Watson AI Lab, IBM Research, MA, USA} 
	\affil[2]{\small Department of Electrical and Computer Engineering, The University of Texas at Austin, USA}
	\affil[ ]{\small \textit \textsuperscript{a} mkocaoglu@utexas.edu \textsuperscript{b}
		shakkott@austin.utexas.edu \textsuperscript{c}
		dimakis@austin.utexas.edu \textsuperscript{d} constantine@utexas.edu \textsuperscript{e} sriram@austin.utexas.edu}
	
	\renewcommand\Authands{ and }
	
	\maketitle

\begin{abstract}
We study the problem of discovering the simplest latent variable that can make two observed discrete variables conditionally independent. The minimum entropy required for such a latent is known as common entropy in information theory. We extend this notion to \renyi common entropy by minimizing the \renyi entropy of the latent variable. To efficiently compute common entropy, we propose an iterative algorithm that can be used to discover the trade-off between the entropy of the latent variable and the conditional mutual information of the observed variables. We show two applications of common entropy in causal inference: First, under the assumption that there are no low-entropy mediators, it can be used to distinguish causation from spurious correlation among almost all joint distributions on simple causal graphs with two observed variables. Second, common entropy can be used to improve constraint-based methods such as PC or FCI algorithms in the small-sample regime, where these methods are known to struggle. We propose a modification to these constraint-based methods to assess if a separating set found by these algorithms is valid using common entropy. We finally evaluate our algorithms on synthetic and real data to establish their performance.
\end{abstract}

\input{introduction}
\input{renyi_common_entropy}
\input{bivariate}
\input{entropic_constraint_based}
\input{experiments}
\input{conclusions}
\bibliography{causalinferenceSimple}
\clearpage
\input{appendix}
\bibliographystyle{plain}

\end{document}

%% file: introduction.tex
\section{Introduction}
\label{sec:intro}
Understanding the causal workings of a system from data is essential in many fields of science and engineering. Recently, there has been increasing interest in causal inference in the machine learning (ML) community. While most of ML has traditionally been relying solely on correlations in the data, it is now widely accepted that distinguishing causation from correlation is useful even for simple predictive tasks. This is because causal relations are more robust to the changes in the dataset and can help with generalization,  while an ML system relying solely on correlations might suffer when these correlations change in the environment the system is deployed in \cite{de2019causal}.

A causal graph is a directed acyclic graph that depicts the causal workings of the system under study \cite{Pearl2009}. Since it indicates the causes of each variable, it can be seen as a qualitative summary of the underlying mechanisms. Learning the causal graph is the first step for most of the causal inference tasks, since inference algorithms rely on the causal structure. Causal graphs can be learned from randomized experiments \cite{Hauser2012a,EberhardtThesis,Shanmugam2015,kocaoglu2017experimental,kocaoglu2019characterization}. In settings where performing experiments are costly or infeasible, one needs to resort to observational methods, i.e., make best use of observational data, potentially under assumptions about the data generating mechanisms.

There is a rich literature on learning causal graphs from observational data \cite{Spirtes2001,zhang2006causal,colombo2014order,Andersson1997,Meek1995a,Shimizu2006,Hoyer2008,verma1992algorithm,Chickering2002}. \emph{Score-based} methods optimize a regularized likelihood function in order to discover the causal graph. Under certain assumptions, these methods are consistent; they obtain a causal graph that is in the correct equivalence class for the given data \cite{meek1997graphical,chickering2002finding}. However, score-based methods are applicable only in the causally sufficient setting, i.e., when there are no latent confounders. A variable is called a latent confounder if it is not observable and causes at least two observed nodes. \emph{Constraint-based} methods directly recover the equivalence class in the form of a mixed graph \cite{Andersson1997,meek1997graphical,Spirtes2001, Meek1995a,zhang2008completeness}: They test the conditional independence (CI) constraints in the data and use them to infer as much as possible about the causal graph. Despite being well-established for graphs with or without latents, constraint-based methods are known to work well only with an abundance of data. Early errors in CI statements might lead to drastically different graphs due to the sequential nature of these algorithms. A third class of algorithms can be described as those imposing assumptions in order to identify the graphs which are otherwise not identifiable \cite{Hoyer2008,Shimizu2006,Peters2014,janzing2009identifying,Kocaoglu2017}. Most of this literature focuses on the cornerstone case of two variables $X,Y$ where constraint and score-based approaches are unable to identify if $X$ causes $Y$ or $Y$ causes $X$, simply because they are indistinguishable without additional assumptions.  This literature contains a wide range of assumptions that we summarize in Section \ref{app:related_work}. 

Information theory has been shown to provide tools that can be useful for causal discovery  \cite{chaves2014inferring,Quinn2015,Kontoyiannis2016,Etesami2016,weilenmann2017analysing,Kocaoglu2017,steudel2015information}. In this work, we explore the uses of \emph{common entropy} for learning causal graphs from observational data. To define common entropy, first consider the following problem:  Given the joint probability distribution of two discrete random variables $X,Y$, we want to construct a third random variable $Z$ such that  $X$ and $Y$ are independent conditioned on $Z$. Without any constraints this can be trivially achieved: Simply picking $Z = X$ or $Z = Y$ ensures that $\ci{X}{Y}{Z}$. However, this trivial solution requires $Z$ to be as \emph{complex} as $X$ or $Y$. We then ask the following question: \emph{is there a simple $Z$ that makes $X,Y$ conditionally independent?} In this work, we use \renyi entropy of the variable as a notion of its complexity. Then the problem becomes identifying $Z$ with the smallest \renyi entropy such that $\ci{X}{Y}{Z}$. Shannon entropy of this $Z$ is called the \emph{common entropy} of $X$ and $Y$~\cite{kumar2014exact}.

We demonstrate two uses of common entropy for causal discovery. The first is in the setting of two observed variables. Suppose we observe 
two correlated variables $X,Y$. Figure \ref{fig:causal_graphs} shows some causal graphs that can induce correlation between $X,Y$. Note that \emph{Latent Graph} differs from the others in that $X$ does not have any causal effect on $Y$. Then distinguishing the latent graph from the others is important to understand whether an intervention on one of the variables will cause a change in the other. We show that if the latent confounder is \emph{simple}, i.e., has small entropy then one can distinguish the latent graph from the triangle and direct graphs using common entropy. To identify the latent graph, we assume that the correlation is not induced only by a simple mediator, which eliminates the mediator graph. We show that this is a realistic assumption using simulated and real data. 

Second, we show that common entropy can be used to improve constraint-based causal discovery algorithms in the small sample regime. For such algorithms, correctly identifying \emph{separating sets}, i.e., sets of variables that can make each pair conditionally independent, is crucial. Our key observation is that, for a given pair of variables common entropy provides us with an information-theoretic lower bound on the entropy of any separating set. Therefore, it can be used to reject incorrect separating sets. We present our modification on the PC algorithm, which is called the \emph{EntropicPC} algorithm. 

To the best of our knowledge, the only result for finding common entropy is given in \cite{kumar2014exact}, where they identify its analytical expression for binary variables. They also note that the problem is difficult in the general case. To address this gap, in Section \ref{sec:common_entropy} we propose an iterative algorithm to approximate common entropy. We also generalize the notion of common entropy to \emph{\renyi common entropy}. 

Our contributions can be summarized as follows:
\begin{itemize}
	\item In Section \ref{sec:common_entropy}, we introduce the notion of \renyi common entropy. We propose a practical algorithm for finding common entropy and prove certain guarantees. Readers interested only in  the applications of common entropy to causal inference can skip this section.
	\item In Section \ref{sec:bivariate}, under certain assumptions, we show that $\rce_0$ common entropy can be used to distinguish latent graph from the triangle and direct graphs in Figure \ref{fig:causal_graphs}. We also show this identifiability result via $\rce_1$ common entropy for binary variables, and propose a conjecture for the general case. In Section \ref{sim:mediator}, we validate one of our key assumptions in real and synthetic data. In Section \ref{sim:conjecture}, we validate our conjecture via real and synthetic data. 
	\item In Section \ref{sec:entropicPC}, we propose \emph{EntropicPC}, a modified version of the PC algorithm that uses common entropy to improve sample efficiency. In Section \ref{sim:EntropicPC}, we demonstrate significant performance improvement for EntropicPC compared to the baseline PC algorithm. We also illustrate that EntropicPC discovers some of the edges missed by PC in ADULT data \cite{Dua2017}. 
	\item In Section \ref{sec:experiments}, in addition to the above, we provide experiments on the performance of our algorithm for finding common entropy, as well as its performance on distinguishing the latent graph from the triangle graph on synthetic data.
\end{itemize}

\textbf{Notation: }
Support of a discrete random variable $X$ is shown as $\mathcal{X}$. $p(.)$ and $q(.)$ are reserved for discrete probability distribution functions (pmfs). $[n]\coloneqq \{1,2,\hdots,n\}$. $p(Y|x)$ is shorthand for the conditional distribution of $Y$ given $X\!=\!x$. Shannon entropy, or entropy in short, is $H(X)\!=\!-\sum_x p(x)\log(p(x))$. \renyi entropy of order $r$ is $H_r(X)\!=\!\frac{1}{1-r}\log\left(\sum_xp^r(x)\right)$. It can be shown that \renyi entropy of order $1$ is identical to Shannon entropy.  $D\!=\!(\mathcal{V},\mathcal{E})$ is a directed acyclic graph with vertex set $\mathcal{V}$ and edge set $\mathcal{E}\subset \mathcal{V}\times \mathcal{V}$. $Pa_i$ represents the set of parents of vertex $X_i$ in the graph and $pa_i$ a specific realization. If $D$ is a causal graph, the joint distribution between the variables (vertices of the graph)  factorizes relative to the graph as $p(x_1,x_2,\hdots,)\!=\!\prod_ip(x_i|pa_i)$. 

\begin{figure}
\centering
\begin{subfigure}[b]{0.24\textwidth}
	\centering
	\begin{tikzpicture}
	\node[draw,shape=circle,below] (x) at (0,0) {$X$};
	\node[draw,shape=circle,below] (y) at (1.4,0) {$Y$};
	\node[draw,shape=rectangle,below] (u) at (0.7,0.5){$Z$};
	\draw (u) edge[->]  (x)
	(u) edge[->] (y);
	\end{tikzpicture}
	\caption{Latent Graph}
\end{subfigure}
\begin{subfigure}[b]{0.24\textwidth}
	\centering
	\begin{tikzpicture}
	\node[draw,shape=circle,below] (x) at (0,0) [circle] {$X$};
	\node[draw,shape=circle,below] (y) at (1.4,0) [circle] {$Y$};
	\node[draw,shape=rectangle,below] (u) at (0.7,0.5) [rectangle] {$Z$};
	\draw (u) edge[->] (x) 
	(u) edge[->] (y)
	(x) edge[->] (y);
	\end{tikzpicture}
	\caption{Triangle Graph}
\end{subfigure}		
\begin{subfigure}[b]{0.24\textwidth}
	\centering
	\begin{tikzpicture}
	\node[draw,shape=circle,below] (x) at (0,0) {$X$};
	\node[draw,shape=circle,below] (y) at (1.4,0) {$Y$};
	\draw (x) edge[->] (y);
	\end{tikzpicture}
	\caption{Direct Graph}	
\end{subfigure}
\begin{subfigure}[b]{0.24\textwidth}
	\centering
    \begin{tikzpicture}
	\node[draw,shape=circle,below] (x) at (0,0) {$X$};
	\node[draw,shape=circle,below] (y) at (2.,0) {$Y$};
	\node[draw,shape=rectangle,below] (m) at (1.,-0.1){$M$};
	\draw (x) edge[->] (m) 	(m) edge[->] (y);
\end{tikzpicture}
\caption{Mediator Graph}	
\end{subfigure}
\caption{Different graphs that explain correlation between the observed $X,Y$. $Z, M$ are unobserved. 
}
\label{fig:causal_graphs}
\end{figure}

%% file: renyi_common_entropy.tex
\section{\renyi Common Entropy}
\label{sec:common_entropy}
We introduce the notion of \renyi common entropy, which generalizes the common entropy of \cite{kumar2014exact}. 
\begin{definition}
	\renyi common entropy of order $r$ or $\rce_r$ common entropy of two random variables $X,Y$ with probability distribution $p(x,y)$ is shown by  $G_r(X,Y)$ and is defined as follows:
	\begin{equation}
	\label{eq:common_entropy}
	\begin{aligned}
	G_r(X,Y)\coloneqq &\underset{q(x,y,z)}{\min} & & H_r(Z)\\
	& \,\,\,\,\,\text{s.t.}
	&&I(X;Y|Z)=0;\sum_z q(x,y,z) = p(x,y),\forall x, y; \sum q(.)=1; q(.)\geq 0\\
	\end{aligned}
	\end{equation}
\end{definition}
\renyi common entropy lower bounds the \renyi entropy of any variable that makes the observed variables conditionally independent. We focus on two special cases: $\rce_0$ and $\rce_1$ common entropies. Among all variables $Z$ such that $\ci{X}{Y}{Z}$, $\rce_0$ common entropy is the logarithm of the minimum number of states of $Z$ and  $\rce_1$ common entropy is its minimum entropy.

In Section \ref{sec:bivariate}, we show that $\rce_0$ common entropy can be used for distinguishing the latent graph from the triangle or direct graphs in Figure \ref{fig:causal_graphs}. Since we expect $\rce_0$ common entropy to be sensitive to finite-sample noise in practice, we focus on $\rce_1$ common entropy. $\rce_1$ common entropy, or simply common entropy, was introduced in \cite{kumar2014exact}, where authors derived the analytical expression for two binary variables. They also remark that finding common entropy for non-binary variables is difficult. We propose an iterative update algorithm to approximate common entropy in practice, by assuming that we have access to the joint distribution between $X,Y$.

\subsection*{LatentSearch: An Algorithm for Calculating $\rce_1$ Common Entropy}
In this section, our objective is to solve a relaxation of the $\rce_1$ common entropy problem in (\ref{eq:common_entropy}). Instead of enforcing conditional independence as a hard constraint of the optimization problem, we introduce conditional mutual information as a regularizer to the objective function. This allows us to discover a trade-off between two factors, the entropy $H(Z)$ of the third variable and the residual dependency between $X,Y$ after conditioning on $Z$, measured by $I(X;Y|Z)$. We then have the loss 
\begin{equation}
\label{eq:loss} \mathcal{L} = I(X;Y|Z) + \beta H(Z).
\end{equation}
Rather than searching over $q(x,y,z)$ and enforcing the constraint $\sum_{z}q(x,y,z)=p(x,y), \forall x,y$, we can search over $q(z|x,y)$ and set $q(x,y,z) =
q(z|x,y)p(x,y)$. Therefore we have
$\mathcal{L} = \mathcal{L}(q(z|x,y))$. The support size of $Z$ determines the number of optimization variables. Proposition 5 in \cite{kumar2014exact} shows that without loss of generality, we can assume $\lvert\mathcal{Z}\rvert\leq \lvert\mathcal{X}\rvert\lvert\mathcal{Y}\rvert$. In general, $\mathcal{L}$ is neither convex nor concave. Although first order methods (e.g., gradient descent) can be used to find a stationary point, as we empirically observe the convergence is slow and the performance is very sensitive to the step size. 
\begin{algorithm}[t]
	\small
	\caption{LatentSearch: Iterative Update Algorithm}
	\begin{algorithmic}[1]
		\STATE \textbf{Input:} Supports of $x,y,z$, $\mathcal{X,Y,Z},$
		respectively. $\beta \geq 0$ used in (\ref{eq:loss}). Observed joint $p(x,y)$. Initialization $q_1(z|x,y)$. Number of iterations $N$.
		\STATE \textbf{Output:} Joint distribution $q(x,y,z)$
		\FOR{$i \in [N]$}
		\STATE \emph{Form the joint:}\\
		$q_i(x,y,z) \leftarrow q_i(z|x,y)p(x,y), \forall x,y,z$.
		\STATE \emph{Calculate:}\\
		$q_i(z|x) \leftarrow \frac{\sum\limits_{y\in\mathcal{Y}} q_i(x,y,z)}{\sum\limits_{y\in\mathcal{Y},z\in\mathcal{Z}} q_i(x,y,z)}, \hspace{0.2in} q_i(z|y) \leftarrow \frac{\sum\limits_{x\in\mathcal{X}} q_i(x,y,z)}{\sum\limits_{x\in\mathcal{X},z\in\mathcal{Z}} q_i(x,y,z)}, \hspace{0.2in} q_i(z) \leftarrow \sum\limits_{x\in\mathcal{X},y\in\mathcal{Y}} q_i(x,y,z)$
		\STATE \emph{Update:}\\
		$q_{i+1}(z|x,y) \leftarrow \frac{1}{F(x,y)}\frac{q_i(z|x)q_i(z|y)}{q_i(z)^{1-\beta}}$, where $F(x,y) = \sum\limits_{z\in\mathcal{Z}} \frac{q_i(z|x)q_i(z|y)}{q_i(z)^{1-\beta}}$.
		\ENDFOR
		\STATE return $q(x,y,z)\coloneqq q_{N+1}(z|x,y)p(x,y)$
	\end{algorithmic}
	\label{iterative_algorithm}
\end{algorithm}
To this end, we propose a multiplicative update algorithm {\em LatentSearch} in Algorithm \ref{iterative_algorithm}. Given $p(x,y)$, \emph{LatentSearch} starts from a random initialization $q_0(z|x,y)$, and at each step $i$ iteratively updates $q_i(z|x,y)$ to $q_{i+1}(z|x,y)$ to minimize the loss (\ref{eq:loss}). Specifically, in the $i^{th}$ step it marginalizes the joint $q_i(x,y,z)$ to get $q_i(z|x), q_i(z|y),$ and
$q_i(z),$ and imposing a scaled product form on these marginals,
updates the joint to return  $q_{i+1}(x,y,z)$. This decomposition and the update rule are motivated by the partial derivatives associated with the Lagrangian of the loss function (\ref{eq:loss}) (See Section \ref{sec:alg_proof}). More formally, as we show in the following theorem, after convergence \emph{LatentSearch} outputs a stationary point of the loss function. For the proof, please see Sections \ref{sec:alg_proof}, \ref{sec:alg_proof2}.  

\begin{theorem}
	\label{thm:stationary}
	The stationary points of LatentSearch are also stationary points of the loss in (\ref{eq:loss}). Moreover, for $\beta = 1$, LatentSearch converges to either a local minimum or a saddle point of (\ref{eq:loss}), unless it is initialized at a local maximum.
\end{theorem}

Therefore, if the algorithm converges to a solution, it outputs either a local minimum, local maximum or a saddle point. We observe in our experiments that the algorithm always converges for $\beta\leq 1$.

For each $\beta$, \emph{LatentSearch} outputs a distribution $q(.)$ from which $H(Z)$ can be calculated. When using \emph{LatentSearch} to approximate $\rce_1$ common entropy, we will run it for multiple $\beta$ values and pick the distribution $q(.)$ with the smallest $H(Z)$ such that $I(X;Y|Z)\leq\theta$ for a practical threshold $\theta$ to declare conditional independence. See Figure \ref{fig:IHplane} for a sample output of \emph{LatentSearch} for multiple $\beta$ values in the $I-H$ plane. See also lines $6-7$ of Algorithm \ref{causal_inference_algorithm} for an algorithmic description.  

%% file: bivariate.tex
\section{Identifying Correlation without Causation via \renyi Common Entropy}\label{sec:bivariate}
Suppose we observe two discrete random variables $X,Y$ to be statistically dependent. Reichenbach’s common cause principle states that $X$ and $Y$ are either causally related, or there is a common cause that induces the correlation.\footnote{We assume no selection bias in this work, which can also induce spurious correlation.} If the correlation is \emph{only} due to a common cause, intervening on either variable will have no effect on the other. Therefore, it is important for policy decisions to identify this case of \emph{correlation without causation}. Specifically, we want to distinguish latent graph from the triangle or direct graphs in Figure \ref{fig:causal_graphs}. Since our goal is not to identify the causal direction between $X$ and $Y$, we use triangle, direct and mediator graphs to refer to either direction. We show that, under certain assumptions, \renyi common entropy can be used to solve this identification problem.  

Our key assumption is that the latent confounders, if they exist, have small \renyi entropy. In other words, in Figure \ref{fig:causal_graphs} $H_r(Z)\leq \theta_r$ for some $\theta_r$. We consider two cases: $\rce_0$ and $\rce_1$ entropies. $H_0(Z)\leq \theta_0$ is equivalent to upper bounding the support size of $Z$. $H_1(Z)\leq \theta_1$ upper bounds the Shannon entropy of $Z$. In general, $H_1(Z)\leq\theta$ can be seen as a relaxation of $H_0(Z)\leq \theta$ as the latter implies the former but not vice verse. Accordingly, we show stronger results for $\rce_0$, whereas we leave the most general identifiability result of $\rce_1$ as a conjecture. We also quantify how small the confounder's \renyi entropy should be for  identifiability.

Note that bounding the \renyi entropy of the latent confounder in the latent graph bounds the \renyi common entropy of the observed variables. Therefore, in order to distinguish the latent graph from the triangle and direct graphs, we need to obtain lower bounds on the \renyi common entropy of a typical pair $X,Y$ when data is generated from the triangle or direct graphs. 

We first establish bounds on the $\rce_0$ common entropy for the triangle and direct graphs, which hold for almost all parametrizations. To measure the fraction of causal models for which our bound is valid, we use a natural probability measure on the set of joint distributions by sampling each conditional distribution uniformly randomly from the probability simplex:
\begin{definition}[Uniform generative model (UGM)]
	\label{def:ugm}
	For any causal graph, consider the following generative model for the joint distribution $p(x_1,x_2,\hdots)=\prod_i p(x_i|pa_i)$, where $X_i\in \mathcal{X}_i,\forall i$: For all $i$ and $pa_i$, let the conditional distribution $p(X_i|pa_i)$ be sampled independently and uniformly randomly from the probability simplex in $\lvert\mathcal{X}_i\rvert$ dimensions. 
\end{definition}	

The following theorem uses the measure induced by UGM to show that for almost all distributions obtained from the triangle or direct graph, $\rce_0$ common entropy of the observed variables is large.
\begin{theorem}
	\label{thm:identifiability}
	Consider the random variables $X,Y,Z$ with supports $[m],[n],[k]$, respectively. Let $p(x,y,z)$ be a pmf sampled from the triangle or the direct graphs according to UGM. Then with probability 1, $G_0(X,Y) = \log(\min\{m,n\})$.
\end{theorem}
Now consider the latent graph where $Z$ is the true confounder. We clearly have that $G_0(X,Y)\leq H_0(Z)$ since $Z$ indeed makes $X,Y$ conditionally independent. In other words, $G_0(X,Y)$ is upper bounded in the latent graph whereas it is lower bounded in the triangle and the direct graphs by Theorem \ref{thm:identifiability}. Therefore, as long as the correlation cannot be explained by the mediator graph, $G_0(X,Y)$ can be used as a parameter to identify the latent graph. In order to formalize this identifiability statement, we need two assumptions with parameters $(r,\theta)$:
\begin{assumption}[$r,\theta$]
	\label{ass:low_entropy_confounder}
	Consider any causal model with observed variables $X,Y$. Let $Z$ represent the variable that captures all latent confounders between $X,Y$. Then $H_r(Z)< \theta$.  
\end{assumption}
\begin{assumption}[$r,\theta$]
	\label{ass:no_highentropy_mediator}
	Consider a causal model where $X$ causes $Y$. If $X$ causes $Y$ only through a latent mediator $Z$, i.e., $X\rightarrow Z\rightarrow Y$, then $H_r(Z)\geq \theta$.
\end{assumption}
Assumption \ref{ass:low_entropy_confounder} states that the collection of latent confounders, represented by $Z$, has to be "simple", which is quantified by its \renyi entropy. This assumption can also be interpreted as relaxing the causal sufficiency assumption by allowing \emph{weak} confounding. Assumption \ref{ass:no_highentropy_mediator} states that if the correlation is induced only due to a mediator, this mediator cannot have low \renyi entropy. Even though this assumption might seem restrictive, we provide evidence on both real and synthetic data in Section \ref{sim:mediator} to show it indeed often holds in practice. We have the following corollary: 
\begin{corollary}
	\label{cor:cardinality}
	Consider the random variables $X,Y$ with supports $[m],[n]$, respectively. For $\theta\!=\!\log(\min\{m,n\})$, latent graph can be identified with probability $1$ under UGM, Assumption \ref{ass:low_entropy_confounder},\ref{ass:no_highentropy_mediator}($0,\theta$).
\end{corollary}

Corollary \ref{cor:cardinality} indicates that, when the latent confounder has less than $\min\{m,n\}$ number of states, we can infer that the true causal graph is the latent graph from observational data under Assumptions \ref{ass:low_entropy_confounder} and \ref{ass:no_highentropy_mediator}. However, using common entropy, we cannot distinguish triangle graph from the direct graph. Also note that the identifiability result holds for \emph{almost all} parametrizations of these graphs, i.e., the set of parameters where it does not hold has Lebesgue measure zero.  

Next, we investigate if $\rce_1$ common entropy can be used for the same goal. Finding $\rce_1$ common entropy in general is  challenging. For binary $X,Y$ we can use the analytical expression of \cite{kumar2014exact} to show that $G_1(X,Y)$ is almost always larger than $H(Z)$ asymptotically for the triangle graph:
\begin{theorem}
	\label{thm:binary_identifiability}
	Consider the random variables $X,Y,Z$ with supports $[2],[2],[k]$, respectively. Let $p(x,y,z)$ be a pmf sampled from the triangle graph according to UGM except $p(z)$, which can be arbitrary. Then $\!\!\!\!\lim\limits_{H(Z)\rightarrow 0}\!\!\!\! \mathbb{P}(G_1(X,Y)\!>\! H(Z))\!\!=\!\!1$, where $\mathbb{P}$ is the probability measure induced by UGM.	
\end{theorem}
\begin{algorithm}[t]
	\small
	\caption{InferGraph: Identifying the Latent Graph}
	\begin{algorithmic}[1]
		\STATE \textbf{Input:} $k:$ Support size of $Z$, $p(x,y)$, $T: I(X;Y|Z)$ threshold, $\{\beta_i\}_{i\in [N]}$, $\theta: H(Z)$ threshold.
		\STATE Randomly initialize $N$ distributions $q_0^{(i)}(z|x,y),i\in[N]$.
		\FOR{$i\in [N]$}
		\STATE $q^{(i)}(x,y,z) \leftarrow$ {\em LatentSearch}($q_0^{(i)}(z|x,y),\beta_i$).
		\STATE Calculate $I^{(i)}(X;Y|Z)$ and $H^{(i)}(Z)$ from $q^{(i)}(x,y,z)$.
		\ENDFOR 
		\STATE $S = \{i:I^{(i)}(X;Y|Z)\leq T\}$.
		\IF{$\min(\{H^{(i)}(Z):i\in S\})>\theta$ or $S=\emptyset$}
		\STATE \textbf{return} Triangle or Direct Graph 
		\ELSE
		\STATE \textbf{return} Latent Graph
		\ENDIF
	\end{algorithmic}
	\label{causal_inference_algorithm}
\end{algorithm}
In Section \ref{app:binary}, we provide simulations for binary and ternary $Z$ to demonstrate the behavior for small non-zero $H(Z)$. Then, we have the following asymptotic identifiability result using common entropy:
\begin{corollary}
	\label{cor:binary}
	For binary $X,Y$ under UGM, Assumption \ref{ass:low_entropy_confounder},\ref{ass:no_highentropy_mediator}($1,\theta$) if the entropy upper bound $\theta$ is known, the fraction of causal models for which latent graph can identified goes to $1$ as $\theta$ goes to $0$.
\end{corollary}

For the general case, we conjecture that when the data is sampled from the triangle or the direct graphs, $G_1(X,Y)$ scales with $\min\{H(X),H(Y)\}$. 
\begin{conjecture}
	\label{conjecture}
	Consider the random variables $X,Y,Z$ with supports $[m],[n],[k]$, respectively. Let $p(x,y,z)$ be a pmf sampled from the triangle or direct graphs according to UGM except $p(z)$, which can be arbitrary. Then, there exists a constant $\alpha\!=\!\Theta(1)$ such that with probability $1-\left(\min\{m,n\}\right)^{-c}$ $G_1(X,Y)\!>\!\alpha\min\{H(X),H(Y)\}$ for some constant $c=c(\alpha)$.
\end{conjecture}
According to Conjecture \ref{conjecture}, we expect that for most of the parametrizations of the triangle and direct graphs, common entropy of the observed variables should be lower-bounded by the entropies of the observed variables, up to a scaling by a constant. It is easy to see that under assumptions similar to those in Corollaries \ref{cor:cardinality} and \ref{cor:binary}, Conjecture \ref{conjecture} implies identifiability of the latent graph. In Section \ref{sec:experiments}, we conduct experiments to support the conjecture and identify $\alpha$. We conclude this section by formalizing how \emph{LatentSearch} can be used in Algorithm \ref{causal_inference_algorithm}, under Assumption \ref{ass:low_entropy_confounder}$(1,\theta)$, \ref{ass:no_highentropy_mediator}$(1,\theta)$. Conjecture \ref{conjecture} suggests that, in Algorithm \ref{causal_inference_algorithm}, we can set $\theta\!=\!\alpha\min\{H(X),H(Y)\}$ for some $\alpha<1$. 

%% file: entropic_constraint_based.tex
\begin{algorithm}[t]
	\caption{EntropicPC (for $F=False$) and EntropicPC-C (for $F=True$)}
	\begin{algorithmic}[1]
		\small
		\STATE \textbf{Input:} CI Oracle $\mathcal{C}$ for $\mathcal{V}=\{X_1,\hdots X_n\}$. Common entropy oracle $\mathcal{B}$. Entropy oracle $H$. Flag $F$.
		\STATE Form the complete undirected graph $D=(\mathcal{V},\mathcal{E})$ on node set $\mathcal{V}$.
		\STATE $l\leftarrow -1$. $maysep(X_i,X_j)\leftarrow True, \forall i,j $.
		\WHILE{$\exists i,j$ s.t. $(X_i,X_j)\in\mathcal{E}$ \textbf{and} $\lvert adj_D(X_i)\backslash \{X_j\}\rvert>l$ \textbf{and} $maysep(X_i,X_j)=True$.}
		\STATE $l\leftarrow l+1$
		\FOR{All $i,j$ s.t. $(X_i,X_j)\in \mathcal{E}$ and $\lvert adj_D(X_i)\backslash \{X_j\}\rvert\geq l$}
		\WHILE{$(X_i,X_j)\in \mathcal{E}$ and $\exists S\subseteq adj_D(X_i)\backslash \{X_j\}$ s.t. $\lvert S\rvert=l $ \textbf{and} $maysep(X_i,X_j)=True$.}
		\STATE Pick a new $S\subseteq adj_D(X_i)\backslash \{X_j\}$ s.t. $\lvert S\rvert=l $. 
		\IF{$\mathcal{C}(X,Y|Z)=True$} 
		\IF{ $H(S)\geq\mathcal{B}(X_i,X_j)$ }
		\STATE $\mathcal{E}\leftarrow \mathcal{E} - \{(X_i,X_j)\}$. 
		\STATE $sepset(X_i,X_j)\leftarrow S$.
		\ELSIF{$F=True$}
		\STATE $maysep(X_i,X_j)\leftarrow False$
		\ELSE
		\IF{$\mathcal{B}(X_i,X_j)\geq 0.8\min\{H(X_i),H(X_j)\}$}
		\STATE $maysep(X_i,X_j)\leftarrow False$
		\ENDIF
		\ENDIF
		\ENDIF
		\ENDWHILE
		\ENDFOR
		\ENDWHILE
		\STATE Orient unshielded colliders according to separating sets $\{sepset(X_i,X_j)\}_{i,j\in [n]}$ \cite{colombo2014order}.
		\STATE Orient as many of the remaining undirected edges as possible by  repeatedly applying the Meek rules \cite{Meek1995a}.
		\STATE \textbf{Return:} $D=(\mathcal{V},\mathcal{E})$.
	\end{algorithmic}
	%\caption{Greedy Algorithm with Quantization}
	\label{entropicPC}
\end{algorithm}
\section{Entropic Constraint-Based Causal Discovery}
\label{sec:entropicPC}
%Constraint-based causal discovery algorithms utilize conditional independence tests to learn as much as plausible about the underlying causal graph from observational data \cite{Spirtes2001}. 
%One of the most common line of work for learning causal graphs from observational data utilizes conditional independence tests: 
A causal graph imposes certain CI relations in the data. Constraint-based causal discovery methods utilize CI statements to reverse-engineer the underlying causal graph. Consider a causal graph over a set $\mathcal{V}$ of observed variables. Constraint-based methods identify a set $S_{X,Y}\subset\mathcal{V}$ as a \emph{separating set} for every pair $X,Y$ if the CI statement $\ci{X}{Y}{S_{X,Y}}$ holds in the data. Starting with a complete graph, edges between pairs are removed if they are separable by some set. Separating sets are later used to orient parts of the graph, which is followed by a set of orientation rules~\cite{Spirtes2001,zhang2008completeness} . %\footnote{Due to space constraints, we refer the reader to for a formal treatment.}.

Despite being grounded theoretically in the large sample limit, in practice these methods require a large number of samples and are very sensitive to noise: An incorrect CI statement early on might lead to a drastically different causal graph at the end due to their sequential nature. Another issue is that the distribution should be faithful to the graph, i.e., any connected pair should be dependent~\cite{Spirtes2001,koller2009probabilistic}.%\footnote{This connectivity is in the sense of d-separation. Please see .} 

To help alleviate some of these issues, we propose a simple modification to the existing constraint-based learning algorithms using common entropy. Our key observation is that the common entropy of two variables provide an information-theoretic lower bound on the entropy of \emph{any} separating set. In other words, common entropy provides us with a necessary condition for a set $S_{X,Y}$ to be a valid separating set: $\ci{X}{Y}{S_{X,Y}}$ only if $H(S_{X,Y})\!\geq\! G_1(X,Y)$. Accordingly, we can modify any constraint-based method to ensure this condition. We only lay out our modifications on the PC algorithm. It can be trivially applied to  other methods such as modern variants of PC and FCI. 

We propose two versions: \emph{EntropicPC} and the conservative version \emph{EntropicPC-C}. In both, $S_{X,Y}$ is accepted only if $H(S_{X,Y})\!\geq\! G_1(X,Y)$. The difference is how they handle pairs $X,Y$ that are deemed CI despite that $H(S_{X,Y})\!<\! G_1(X,Y)$. EntropicPC-C concludes that the data for $X,Y$ is unreliable and simply declares them non-separable by any set. EntropicPC only does this when common entropy is large, i.e., $G_1(X,Y)\!\geq\! 0.8\min\{H(X),H(Y)\}$; otherwise it searches for another set that may satisfy $H(S_{X,Y})\!\geq\! G_1(X,Y)$. $0.8$ is chosen based on our experiments in Section \ref{sec:experiments}. 

We provide the pseudo-code in Algorithm \ref{entropicPC}.  %adapting from PC version of \cite{colombo2014order}.  %See, for example,  for the pseudo-code for PC and an order-independent version, which we used in our simulations using \emph{pcalg} package in R~\cite{pcalg1,pcalg2}. 
It is easy to see that both algorithms are %Note that 
sound in the sample limit. %EntropicPC-C and PC are indeed identical. %, whereas EntropicPC can be less informative. %, and EntropicPC-C and PC are identical. We show in Section \ref{sec:experiments} that EntropicPC provides significant improvements with finite data. 
The case of $S\!=\!\emptyset$ in \textbf{line} $\mathbf{10}$ is of special interest. %In practice, it is not reasonable to assign 
Setting $H(\emptyset)\!=\!0$ is not reasonable with finitely many  %number of 
samples since the %: Marginally independent variables will have nonzero 
common entropy of independent variables will not be estimated as exactly zero. %due to finite number of samples. 
To address this, in simulations $H(\emptyset)$ is set to $0.1\min\{H(X),H(Y)\}$ in line $10$. This and the choice of $0.8$ as the coefficient in \textbf{line} $\mathbf{16}$ can be seen as hyper-parameters to be tuned. %that needs to be chosen for each dataset.

%% file: experiments.tex
\section{Experiments}
\label{sec:experiments}
\subsection{Performance of LatentSearch} 
\label{sim:latentsearch}
We evaluate how well \emph{LatentSearch} performs by generating data from the latent graph and comparing the entropy it recovers with the entropy of the true confounder $Z$ in Figure \ref{fig:latentsearch_performance}. In the generated data, we ensure $H(Z)$ is bounded above by $1$ for all $n$. This makes the task harder for the algorithm for larger $n$. The left axis shows the fraction of times \emph{LatentSearch} recovers a latent with entropy smaller than $H(Z)$. The right axis shows the worst-case performance in terms of the entropy gap between the algorithm output and true entropy. We generated $100$ random distributions for each $n$. The same number of iterations is used for the algorithm for all $n$. As expected, performance slowly degrades as $n$ is increased, since $H(Z)<1,\forall n$. We conclude that \emph{LatentSearch} performs well within the range of $n$ values we use in this paper. Further research is needed to make \emph{LatentSearch} adapt to $n$.

\subsection{Validating Assumption \ref{ass:no_highentropy_mediator}: No Low-Entropy Mediator }
\label{sim:mediator}
We conducted synthetic and real experiments to validate the assumption that, in practice, it is unlikely for cause-effect pairs to only have low-entropy mediators. First, in Figure \ref{fig:mediator_synthetic} we generated data from $X\!\rightarrow\!Z\!\rightarrow\! Y$ and evaluated $H(Z)$. $p(X)$ is sampled uniformly from the probability simplex. $p(Z|x), \forall x$ are sampled from Dirichlet with parameter $\alpha_{Dir}$. It is observed that mediator entropy scales with $\log_2(n)$ for all cases. This supports Assumption \ref{ass:no_highentropy_mediator} by asserting that for most causal models, unless the mediator has a constant number of states, its entropy is close to $H(X),H(Y)$.

Second, in Figure \ref{fig:tuebingen}, we run \emph{LatentSearch} on the real cause-effect pairs from Tuebingen dataset \cite{mooij2016distinguishing}. Our goal is to test if the causation can be solely due to low-entropy mediators: If it is, then common entropy should be small since mediator can make the observed variables conditionally independent. We used different thresholds for conditional mutual information for declaring two variables conditionally independent. Investigating typical $I-H$ plots for this dataset (see (b)), we conclude that $0.001$ is a suitable CMI threshold for this dataset. From the empirical cdf of $\alpha\coloneqq \frac{G_1(X,Y)}{\min\{H(X),H(Y)\}}$ across the dataset, we identified that for most pairs $G_1(X,Y) \geq 0.8 \min\{H(X),H(Y)\}$. This indicates that if the causation is solely due to a mediator, it must have entropy of at least $0.8\min\{H(X),H(Y)\}$.
\begin{figure}[t!]
	\centering
	\begin{subfigure}[b]{0.32\textwidth}
		\includegraphics[width=\textwidth]{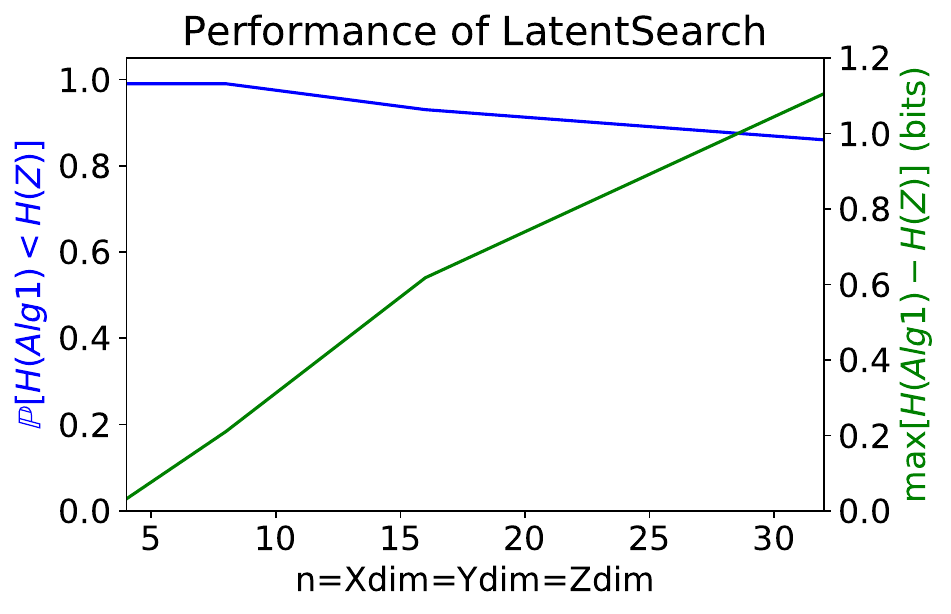}
		\caption{}
		\label{fig:latentsearch_performance}
	\end{subfigure}
	\begin{subfigure}[b]{0.29\textwidth}
		\includegraphics[width=\textwidth]{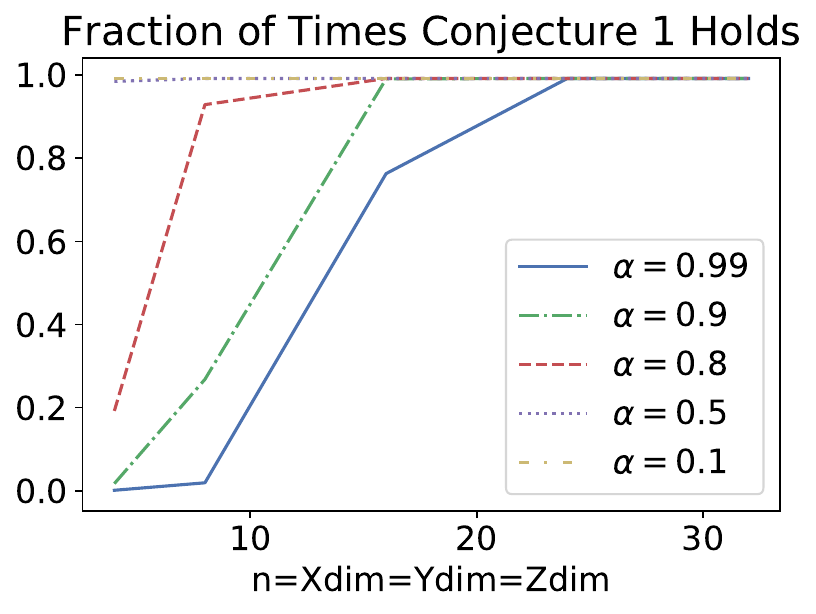}
		\caption{}
		\label{fig:synthetic_conjecture}
	\end{subfigure}
	\begin{subfigure}[b]{0.3\textwidth}
		\includegraphics[width=\textwidth]{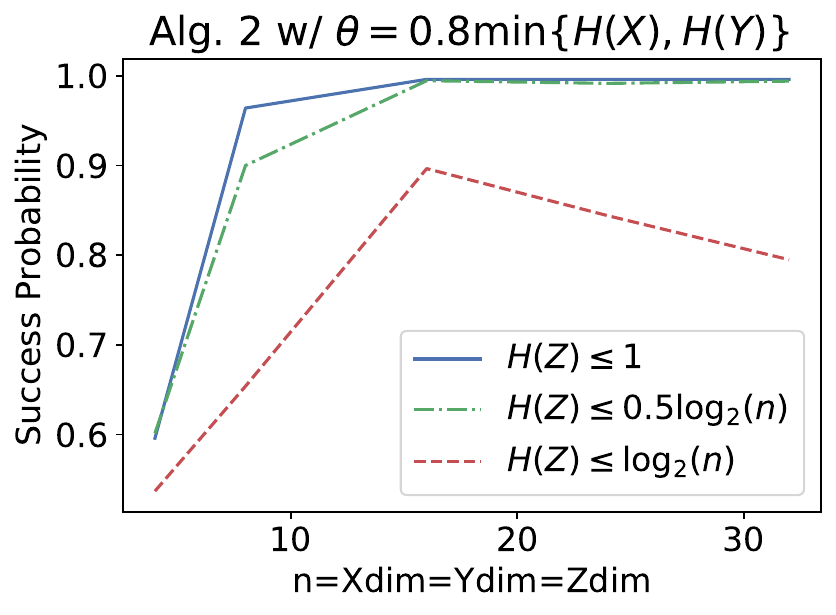}
		\caption{}
		\label{fig:synthetic_identifiability}
	\end{subfigure}
	\caption{\textbf{(a)} [Section \ref{sim:latentsearch}] Performance of \emph{LatentSearch} (Alg. 1) on synthetic data.  \textbf{(b)} [Section \ref{sim:conjecture}] Fraction of times Conjecture \ref{conjecture} holds in synthetic data for different  values of 
		$\alpha$.  \textbf{(c)} [Section \ref{sim:bivariate}] Performance of Algorithm \ref{causal_inference_algorithm} on synthetic data for different confounder entropies.}
	\label{fig:first_panel}
\end{figure}

\begin{figure}[t!]
	\centering
	\begin{subfigure}[b]{0.31\textwidth}
		\includegraphics[width=\textwidth]{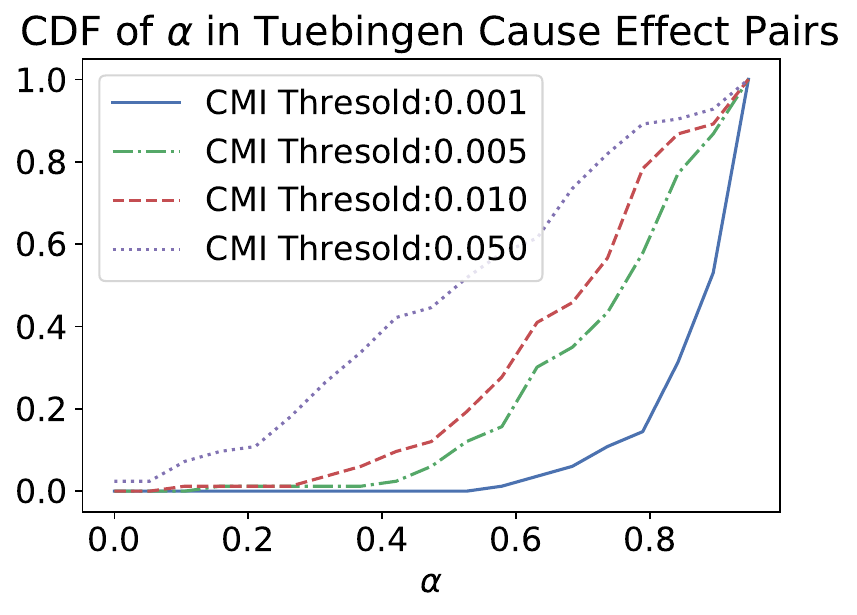}
		\caption{ }
		\label{fig:tuebingen}
	\end{subfigure}
	\begin{subfigure}[b]{0.31\linewidth}
		\includegraphics[width=\textwidth]{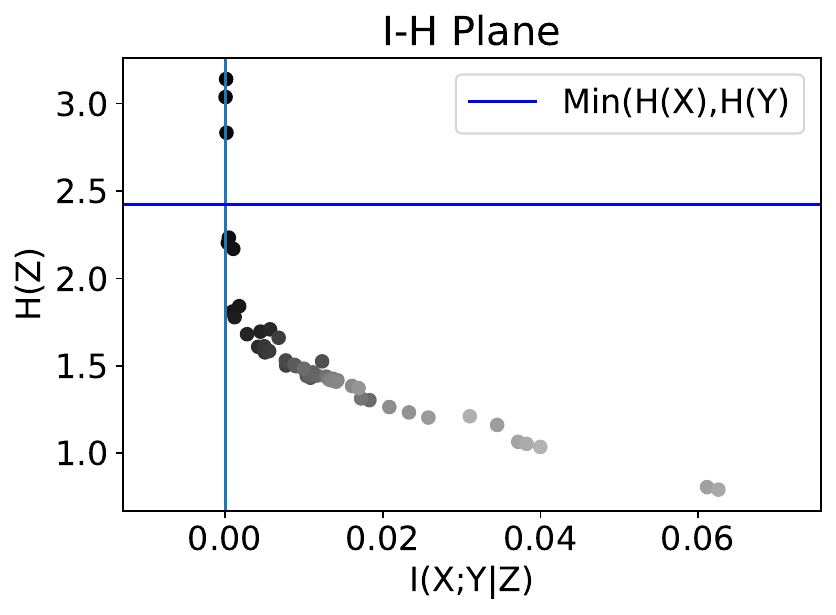}
		\caption{}
		\label{fig:IHplane}
	\end{subfigure}
	\begin{subfigure}[b]{0.31\textwidth}
		\includegraphics[width=\textwidth]{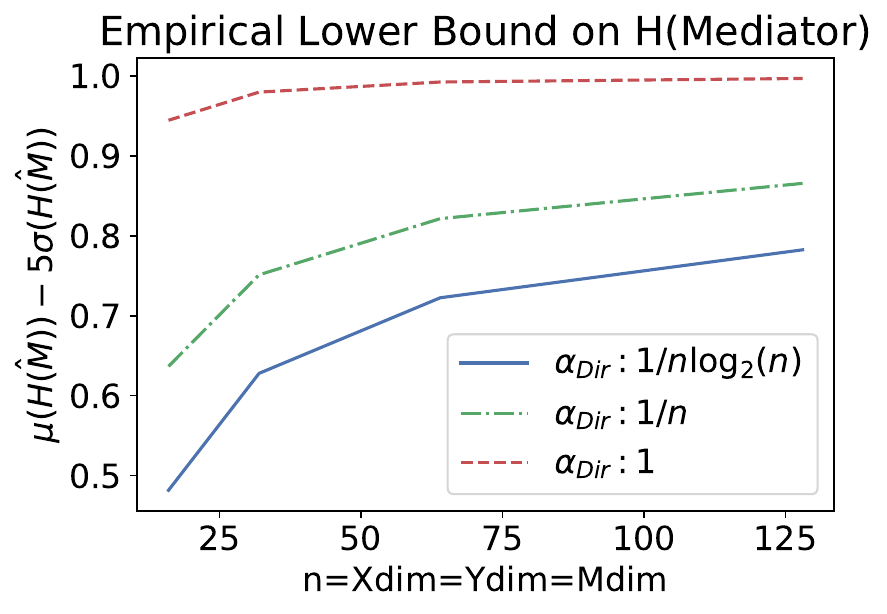}
		\caption{ }
		\label{fig:mediator_synthetic}
	\end{subfigure}
	\caption{\textbf{(a)} [Section \ref{sim:conjecture},\ref{sim:mediator}] Empirical cumulative distribution function for $\alpha$ from Conjecture \ref{conjecture} from Tuebingen Data for various conditional mutual information thresholds. 
		\textbf{(b)} Tradeoff curve discovered by \emph{LatentSearch} for pair 7 of Tuebingen. Each point is the output of \emph{LatentSearch} for a different value of $\beta$. The sharp elbow around $I(X;Y|Z)=0$ hints that the suitable conditional mutual information threshold for this pair is $\approx 0.001$. \textbf{(c)} Entropy of mediator in synthetic data. $y-$axis shows the mean minus $5$ times the st-dev. of entropy as a fraction of $\log_2(n)$ as an empirical lower bound. Results show that, for almost all models, entropy of the mediator scales with $\log_2(n)$.}
	\label{fig:second_panel}
\end{figure}
\subsection{Validating Conjecture \ref{conjecture}} 
\label{sim:conjecture}
In Figure \ref{fig:synthetic_conjecture}, for each $n$ we sample $1000$ distributions from the triangle graph with $H(Z)\leq 1$. Even with a $1$ bit latent, we observe that $G_1(X,Y)$ scales with $\min\{H(X),H(Y)\}$. In fact the result hints at the even stronger statement that for any $\alpha$, $\exists n_0(\alpha)$ where the conjecture holds $\forall n\geq n_0(\alpha)$. 

Experiments with the Tuebingen data, as explained in Section \ref{sim:mediator} and Figure \ref{fig:tuebingen}, similarly supports Conjecture \ref{conjecture}: If a pair is causally related, common entropy typically scales with $\min\{H(X),H(Y)\}$.  

\subsection{Identifiability via Algorithm \ref{causal_inference_algorithm}}
\label{sim:bivariate}
In Figure \ref{fig:synthetic_identifiability} we uniformly mixed data from triangle and latent graphs with different entropy bounds on the confounder. Since in a practical problem, true $H(Z)$ will not be available, we used $0.8\min\{H(X),H(Y)\}$ as the entropy threshold $\theta$ in Algorithm \ref{causal_inference_algorithm}. The results indicate that even if the true $H(Z)$ scales with $n$ (e.g., $0.5\log_2(n)$), we can still distinguish latent from triangle graph. $\leq$ can be interpreted as $\approx$ due to our sampling method of $p(z)$, as detailed in Section \ref{app:sim_details}. % in synthetic data.

\subsection{Evaluating EntropicPC}
\label{sim:EntropicPC}
In this section, to illustrate the performance improvement provided by using common entropy, we compare the order-independent version of PC algorithm~\cite{colombo2014order} with our proposed modification EntropicPC and its conservative version EntropicPC-C as described in Algorithm \ref{entropicPC} on both synthetic graphs and data generated from them and on ADULT dataset. We use \emph{pcalg} package in R~\cite{pcalg1,pcalg2} 
\begin{figure}[t]
	\centering
	\begin{subfigure}[b]{0.3\textwidth}
		\includegraphics[width=\textwidth]{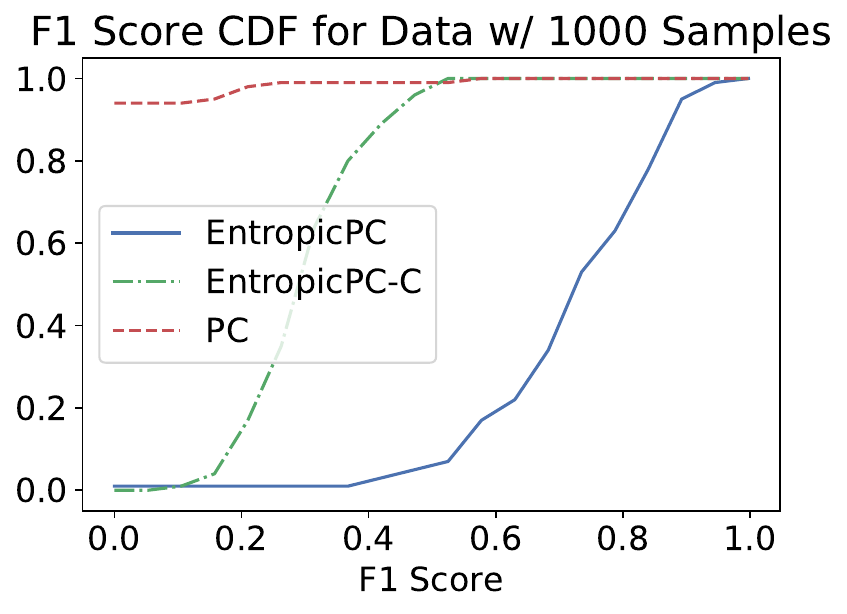}
		\caption{ }
		\label{fig:F1_n10N1000_cum}
	\end{subfigure}
	\begin{subfigure}[b]{0.3\textwidth}
		\includegraphics[width=\textwidth]{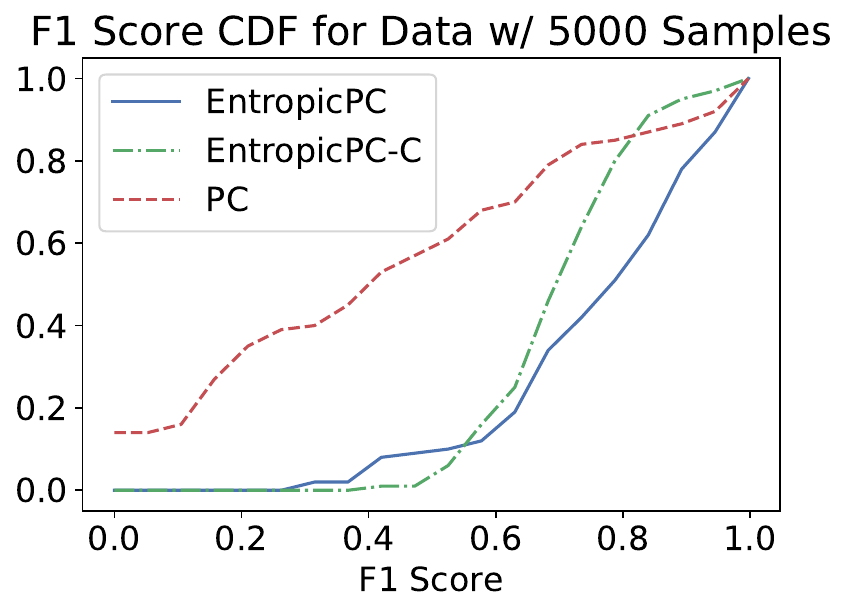}
		\caption{ }
		\label{fig:F1_n10N5000_cum}
	\end{subfigure}
	\begin{subfigure}[b]{0.3\textwidth}
		\includegraphics[width=\textwidth]{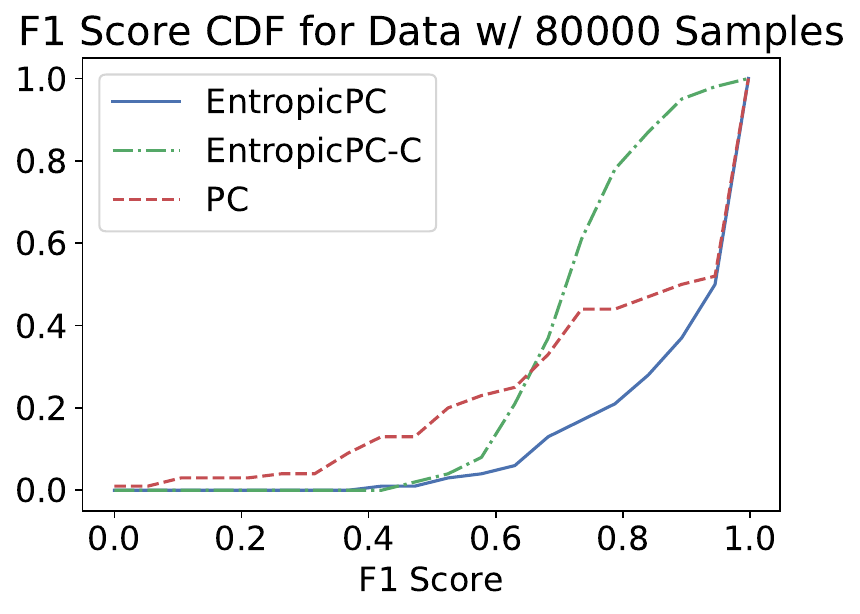}
		\caption{ }
		\label{fig:F1_n10N80000_cum}
	\end{subfigure}
	\begin{subfigure}[b]{0.3\textwidth}
		\includegraphics[width=\textwidth]{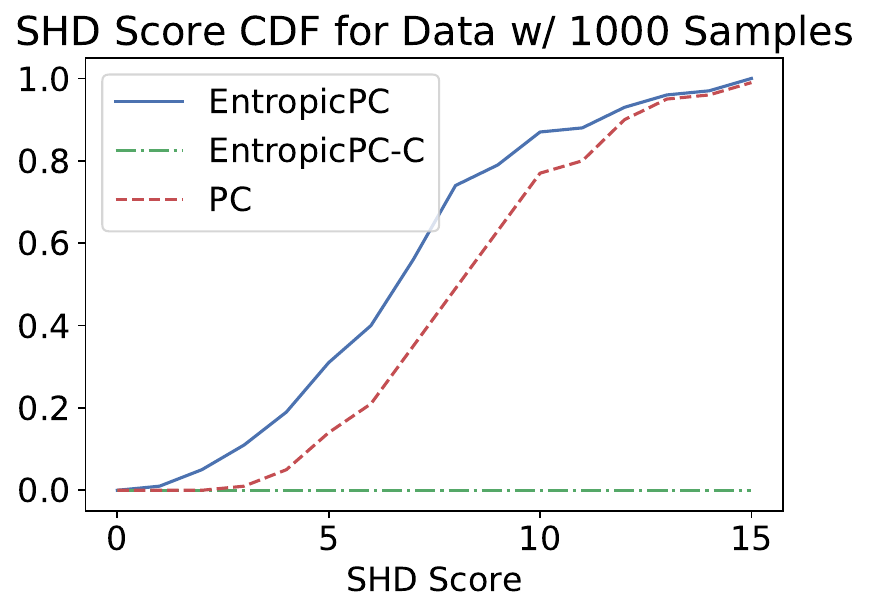}
		\caption{ }
		\label{fig:SHD_n10N1000_cum}
	\end{subfigure}
	\begin{subfigure}[b]{0.3\textwidth}
		\includegraphics[width=\textwidth]{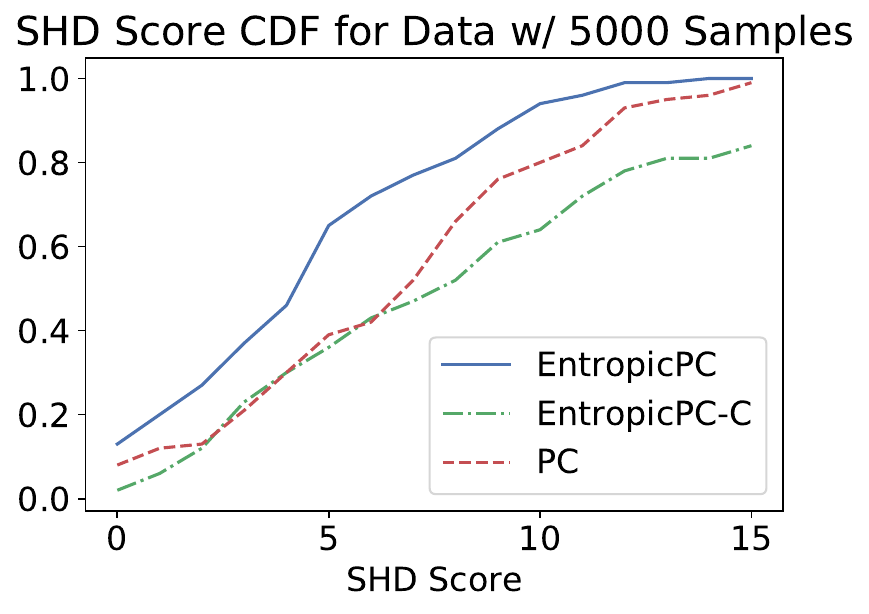}
		\caption{ }
		\label{fig:SHD_n10N5000_cum}
	\end{subfigure}
	\begin{subfigure}[b]{0.3\textwidth}
		\includegraphics[width=\textwidth]{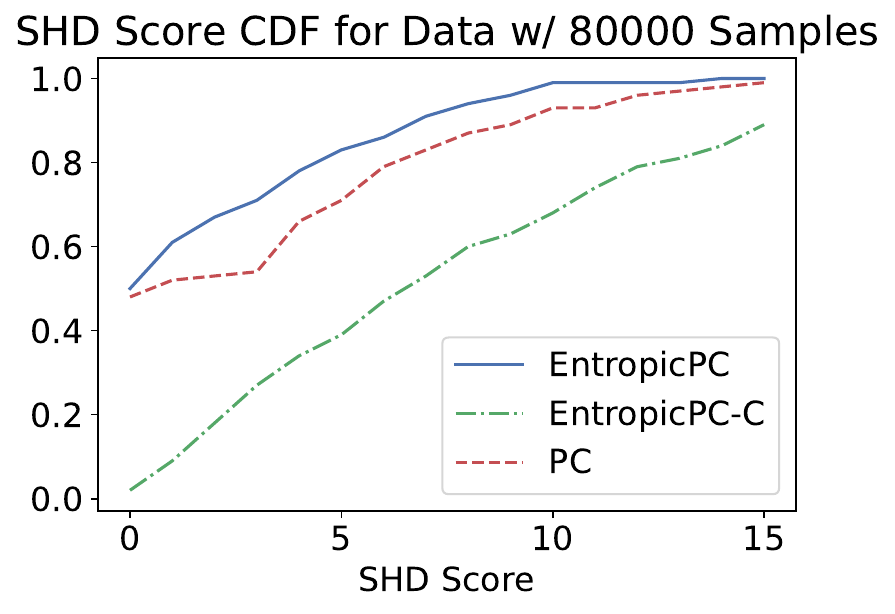}
		\caption{ }
		\label{fig:SHD_n10N80000_cum}
	\end{subfigure}
	\caption{[Section \ref{sim:EntropicPC}] Performance of EntropicPC, EntropicPC-C and PC in synthetic data.}
	\label{fig:entropicpc_synthetic}
\end{figure}

\textbf{Synthetic Graphs/Data: } In Figure \ref{fig:entropicpc_synthetic} we randomly sampled joint distributions from $100$ Erd\"{o}s-\renyi graphs on $10$ nodes (see Section \ref{sec:appEntropicPC} for $3$-node graphs) each with $8$ states, and obtained datasets with $1k,5k,80k$ samples. $80k$ is chosen to mimic the infinite-sample regime where we expect identical performance with PC. \textbf{(a), (b), (c)} shows the empirical cumulative distribution function (CDF) of F1 score for identifying edges in either direction correctly (skeleton). Since $F_1\!=\!1$ is ideal, the best-case CDF is a step function at $1$. Accordingly, lower the CDF curve the better. \textbf{(d), (e), (f)} shows the empirical CDF of Structural Hamming Distance (SHD)~\cite{tsamardinos2006max} between the true essential graphs, and the output of each algorithm using the same synthetic data. Since $SHD\!=\!0$ is ideal, the best-case CDF is a step function at $0$, Accordingly, higher the CDF curve the better. EntropicPC provides very significant improvements to skeleton discovery as indicated by the F1 score and moderate improvement to identifying the true essential graph as indicated by SHD in the small-sample regime.  

\textbf{On ADULT Dataset: } We compare the performance of EntropicPC with the baseline PC algorithm that we used for our modifications. Both causal graphs are shown in Figure \ref{fig:theadult}. Note that the bidirected edges represent undirected, i.e., unoriented edges.

Entropic PC identifies the following additional edges that are missed by the PC algorithm: It discovers that \emph{Marital Status, Relationship, Education, Occupation} causes \emph{Salary}. Even though there is no ground truth, at the very least, we expect \emph{Education} and \emph{Occupation} to be causes of \emph{Salary}, whereas PC outputs \emph{Salary} as an isolated node, implying it is not caused by any of the variables. We believe the reason is that for every neighbor, there is some conditioning set and a configuration with very few number of samples. This can easily be interpreted as conditional independence by a CI tester. EntropicPC alleviates this problem by concluding that there is significant dependence between these variables by checking their common entropy. From both real and synthetic simulations, we conclude that common entropy is more robust to small number of samples than CI testing.

Both algorithms seems to suffer from unfaithful data - \emph{sex} is not required to separate \emph{marital-status} and \emph{occupation} whereas we expect it to since it should be a source node. This drives both algorithms to orient \emph{sex} as a collider. 
\begin{figure}[t]
	\vskip 0.2in
	\centering
	\begin{subfigure}[b]{0.4\textwidth}
		\includegraphics[width=\textwidth]{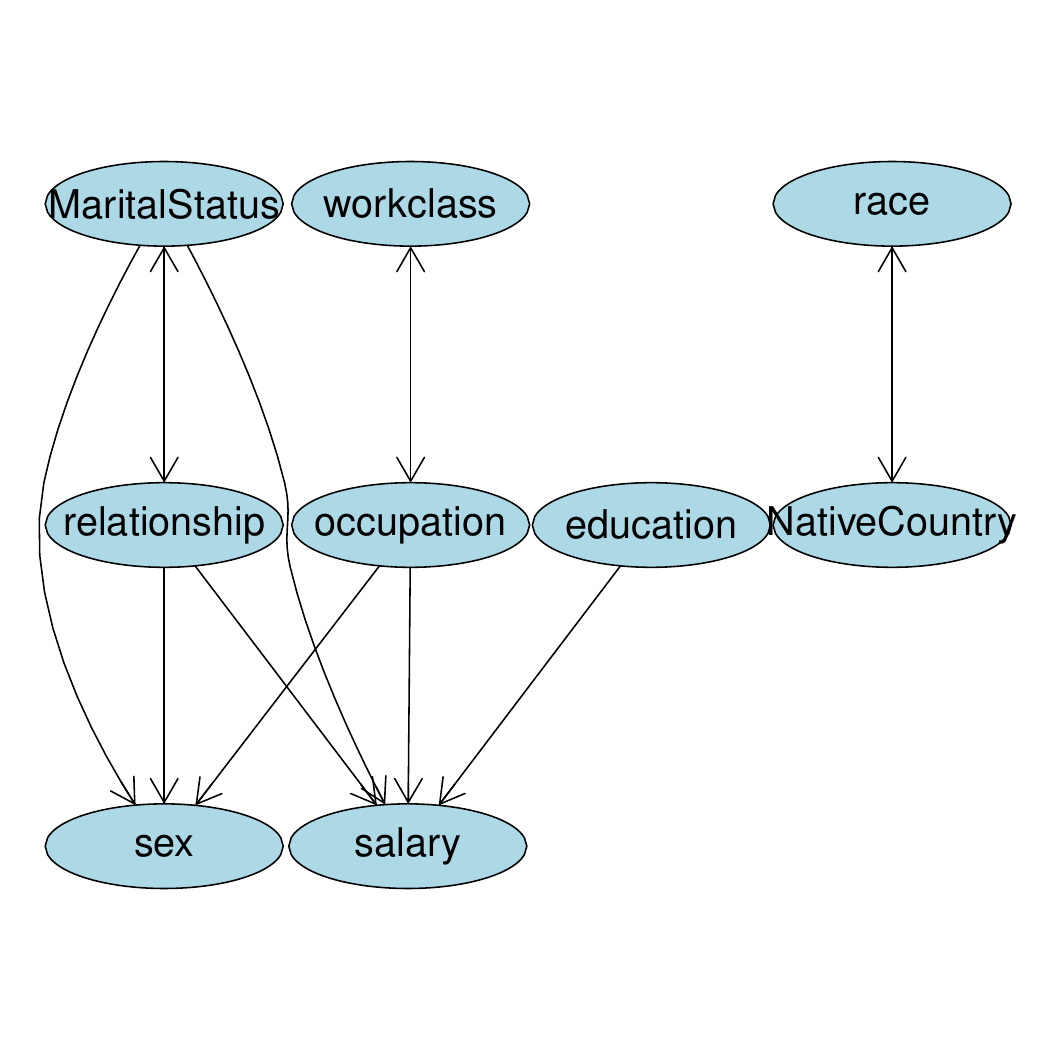}
		\caption{ }
		\label{fig:adult_entropicPC}
	\end{subfigure}
	\hspace{0.2in}
	\begin{subfigure}[b]{0.4\textwidth}
		\includegraphics[width=\textwidth]{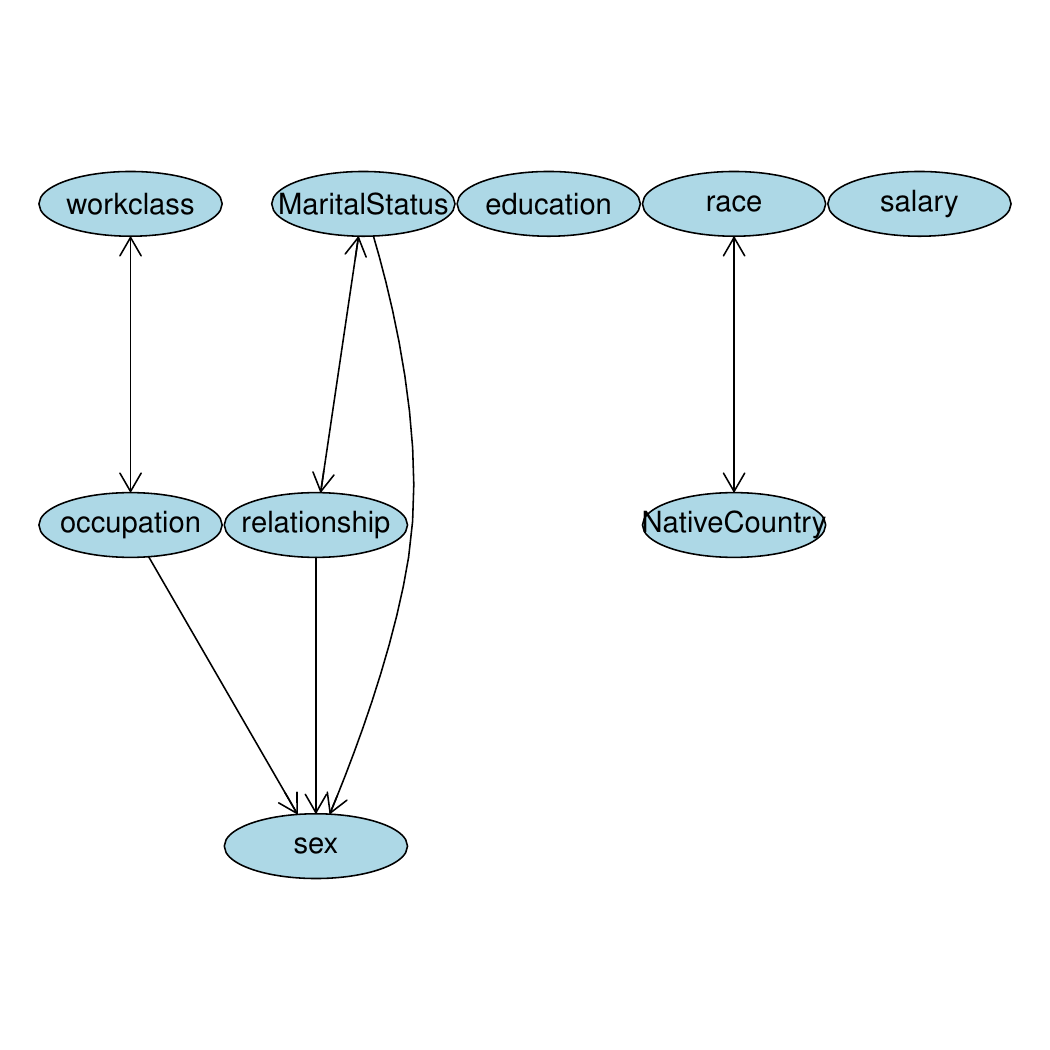}
		\caption{ }
		\label{fig:adult_PC}
	\end{subfigure}
	\caption{[Section \ref{sim:EntropicPC}] Output of \textbf{(a)} EntropicPC and \textbf{(b)} PC in ADULT Data. }
	\label{fig:theadult}
\end{figure}

%% file: conclusions.tex
\section{Conclusions}
In this paper, we showed that common entropy, an information-theoretic quantity, can be used for causal discovery. We introduced the notion of \renyi common entropy and proposed a practical algorithm for approximately calculating $\rce_1$ common entropy. Next, we showed theoretically that common entropy can be used to identify if the observed variables are causally related under certain assumptions. Finally, we showed that common entropy can be used to improve the existing constraint-based causal discovery algorithms. We evaluated our results on synthetic and real data.  

\section*{Acknowledgements}
This work was supported in part by NSF Grants SATC 1704778, CCF 1763702, 1934932, AF 1901292, 2008710, 2019844, 1731754, 1646522, 1609279, ONR, ARO W911NF-17-1-0359, research gifts by Western Digital, WNCG IAP, computing resources from TACC, the Archie Straiton Fellowship.

%% file: appendix.tex
\section{Appendix}
\subsection{Detailed Background}
\label{app:background}
Let $D=(\mathcal{V},\mathcal{E})$ be a directed acyclic graph on the set of vertices $\mathcal{V} = \{V_1,V_2,\hdots, V_n\}$ with directed edge set $\mathcal{E}$. Each directed edge $e_k\in \mathcal{E}$ is a tuple $e_k=(V_i,V_j)$. Let $\mathbb{P} $ be a joint distribution over a set of variables labeled by $\mathcal{V}$. $D$ is called a valid Bayesian network for the distribution $\mathbb{P}$, if $\mathbb{P}$ factorizes with respect to the graph $D$ as $\mathbb{P}(V_1,V_2,\hdots V_n) = \prod_{i} \mathbb{P}(V_i|pa_i)$, where $pa_i$ are the set of parents of vertex $V_i$ in graph $D$. In a Bayesian network $D$ that is valid for $\mathbb{P}$, if three vertices $X,Y,Z$ satisfy a graphical criterion called the \emph{d-separation} on $D$, then $\ci{X}{Y}{Z}$ in $\mathbb{P}$. The faithfulness assumption allows us to infer dependence relations based on d-separation: $\mathbb{P}$ is said to be faithful to graph $D$ when the following holds: Any three variables $X,Y,Z$ that are not d-separated are conditionally dependent, i.e., $\nci{X}{Y}{Z}$ in $\mathbb{P}$. 

Note that the edges in a Bayesian network do not carry a physical meaning: They simply indicate how a joint distribution can be factorized. \emph{Causal Bayesian networks} (or causal graphs)~\cite{Pearl2009}  on the other hand capture the causal relations between variables: They extend the notion of Bayesian networks to different experimental, the so called interventional settings. \emph{An intervention} is an experiment that changes the workings of the underlying system and sets the value of a variable, shown as $do(X=x)$. Causal Bayesian networks allow us to calculate the joint distributions under these experimantal conditions, called the interventional distributions\footnote{For a formal introduction to Pearl's framework please see \cite{Spirtes2001, Pearl2009}.}.  

In this paper, we work with the causal graphs given in Figure \ref{fig:causal_graphs}. From the d-separation principle, we see that the latent graph satisfies $\ci{X}{Y}{Z}$, whereas under the faithfulness condition, $\nci{X}{Y}{Z}$ in the triangle graph or the direct graph. Checking the existence of such a latent variable can help us recover the true causal graph as we discover in the next sections.  We work with discrete ordinal or categorical variables. Suppose the support sizes of the observed variables $X$ and $Y$ are $m$ and $n$, respectively. The joint distribution can be represented with an $m\times n$ non-negative matrix whose entries sum to $1$. We assume that we have access to this joint distribution. 

We use $[n]$ to represent the set $\{1,2,\hdots, n\}$ for any $n\in\mathbb{N}$. Capital letters represent random variables, lowercase letters represent realizations\footnote{In some proofs, $x_i$ is used to represent the probability that the variable $X$ takes the value $i$ for simplicity.}. Letters $X,Y$ are reserved for the observed variables, whereas $Z$ is used for the latent variable. To represent the probability mass function over three variables $X,Y,Z$, we use $p(x,y,z)\coloneqq \mathbb{P}(X = x, Y=y, Z=z)$ and similarly for any conditional $p(z|x,y)\coloneqq \mathbb{P}(Z=z | X=x,Y=y)$. For a function $q(x,y,z)$ that is understood to be a probability mass function, we use shorthand notation for marginals and conditionals such as $q(x,y)$ and $q(x|z)$ to represent the functions obtained from $q(x,y,z)$ via standard operations on probability distributions. Lowercase boldface letters are used for vectors and uppercase boldface letters are used for matrices. We also use $p(Z|x,y)$ to represent the conditional distribution $\mathbb{P}(Z|X=x,Y=y)$ (Similarly for $p(Z|x),p(Z|y)$). $\text{card}(X)$ stands for the support size of $X$. \renyi entropy of order $\alpha$ of a random variable $X$ is defined as $H_\alpha(X)=\frac{1}{1-\alpha}\log{\sum_i p_i^\alpha}$. \renyi entropy of order $0$ gives the support size of a random variable. It can be shown that in the limit as $\alpha\rightarrow 1$, \renyi entropy becomes Shannon entropy, defined as $H_1(X)=-\sum_{x}p(x)\log_2(p(x))$ in bits. In a graph $D$ with nodes labeled as $\{X_i\}_i$, $pa_i$ stands for the set of parents of $X_i$ in $D$. $Dir(\alpha)$ stands for Dirichlet distribution with parameter $\alpha$.

\subsection{Detailed Related Work}
\label{app:related_work}
\noindent \textbf{Latent Variable Discovery:}
Latent variables have been used to model and explain dependence between observed variables in different communities under different names. Probabilistic latent semantic analysis (pLSA) \cite{hofmann1999probabilistic} aims at constructing a variable that explains dependence. However the objective is not to minimize entropy of the constructed variable. Latent Dirichlet allocation (LDA) is another framework which is widely used in topic modeling \cite{blei2003latent,arora2013practical}. Although LDA encourages sparsity of topics, this does not correspond to minimizing the support size of the constructed latent variable. Factorizing the joint distribution matrix between two observed variables via NMF with generalized KL divergence loss recovers solutions to the pLSA problem \cite{gaussier2005relation}. Similar to pLSA, NMF does not have an incentive to discover low-entropy latents. 

Perhaps the most relevant to ours in the machine learning literature are the two papers in the Bayesian setting \cite{brand1999entropic,shashanka2008sparse}. They use low-entropy priors on the latent variable's distribution while performing inference. However their approach is different and their methods cannot be used to discover the tradeoff between conditional mutual information and the entropy of the latent variable. In \cite{shashanka2008sparse}, the authors use low-entropy prior as a proxy for discovering latent factors with sparse support.

Finding the latent variable with smallest entropy that renders the two observed variables conditonally independent is closely related to some of the problems in information theory: Wyner's common information \cite{wyner1975common} is defined as the minimum rate of the source from which the observed variables $X, Y$ can be reconstructed using additional random bits. Wyner allows multiple channel uses and is interested in the approximate reconstruction of the observed joint distribution. This can be seen as approximate reconstruction of the joint distribution when we raise the dimension of $X$ and $Y$ via cartesian product with itself. \cite{kumar2014exact} considers finding the source with the minimum rate for the exact recovery of the observed joint distribution, but still in the asymptotic regime of multiple channel uses. They also introduce the notion of common entropy and obtain an analytical expression for binary variables, which we utilize in this work. 

\noindent \textbf{Learning Causal Graphs with Latents:} Learning causal graphs with latent variables has been extensively studied in the literature. In graphs with many observed variables, some of the edges can be recovered from the observational data (for example through algorithms that employ conditional independence (CI) tests such as IC* \cite{Pearl2009} and FCI \cite{Spirtes2001}). However, latent variables make the CI tests less informative, by inducing spurious correlations between the observed variables. For example for the graphs in Figure \ref{fig:causal_graphs} CI tests on the observed variables is not informative of the causal structure.

Identifiability of causal structures without latent variables from data has been studied extensively in the literature under various assumptions \cite{Hoyer2008,Peters2011,Mooij2010,Peters2014,Peters2016,bahadori2017causal,Etesami2016,Kocaoglu2017}. Our approach can be seen as an extension of \cite{Kocaoglu2017}: There, the authors assume that the exogenous variables have small \renyi entropy and suggest an algorithm to distinguish the causal graph $X\rightarrow Y$ from $X \leftarrow Y$. However, their approach cannot be used in the presence of latent variables. In the presence of latents \cite{cai2018causal} considers a setup similar to \cite{Kocaoglu2017}, where the hidden variable has small support size, however also assumes the mapping to the hidden variable is deterministic. In \cite{janzing2011detecting}, authors identify a condition on $p(Y|X)$ which implies that there does not exist any latent variable $Z$ with small support which can make $X,Y$ conditionally independent. For discrete variables, this assumption implies that the conditionals $p(Y|x)$ lie on the boundary of the probability simplex, which corresponds to the joint probability matrix to be sparse in a structured way. In the continuous variable setting, \cite{sgouritsa2013identifying} propose using kernel methods to detect latent confounders. \cite{weilenmann2017analysing} and \cite{chaves2014inferring} analyzes the discoverability of causal structures with latents using the entropic vector of the variables. Finally, related work also includes \cite{janzing2009identifying} and \cite{liu2016causal}, where the authors extend the additive noise model based approach in \cite{Hoyer2008} to the case with a latent confounder. Algebraic geometry can be used to distinguish causal graphs as the set of distributions that can be encoded by a graph correspond to different algebraic varieties. However, these methods in general are not scalable beyond a few variables and a few number of states \cite{lee2017causal}. Authors in \cite{janzing2010causal} propose using Kolmogorov complexity of the causal model and declare the graph with smaller complexity to be the true graph.  \cite{Kaltenpoth2019} uses description length as a proxy to Kolmogorov complexity to identify the latent confounders. In \cite{steudel2015information}, the authors use information inequalities to infer which subsets of a set of observed variables must have latent confounders, along with an associated lower bound on the entropy of these confounders. In our setting of two observed variables, this gives the trivial bound of $H(Z)\geq I(X;Y)$ for any latent confounder $Z$. 

\subsection{Proof of Theorem \ref{thm:stationary}: Stationarity}
\label{sec:alg_proof}
In this section, we show the first part of Theorem \ref{thm:stationary}, i.e., that the stationarity points of the algorithm are also stationary points of the given loss function.
We write the objective function more explicitly in terms of the optimization variables $q(z|x,y)$:
\begin{align}
\mathcal{L}(q(\cdot|\cdot,\cdot)) &= \sum\limits_{x,y,z} q(x,y,z)\log\left(\frac{q(x,y|z)}{q(x|z)q(y|z)}\right)- \beta \sum\limits_{z}q(z)\log(q(z))\\
&\hspace{-0.3in}=  \sum\limits_{x,y,z} p(x,y)q(z|x,y)\log\left(\frac{q(z|x,y)}{q(z|x)q(z|y)}\right)+( 1 - \beta) \sum\limits_{z}q(z)\log(q(z))+ I(X;Y)
\end{align}
by Bayes rule and assuming that $q(z|x,y)$ and $p(x,y)$ are strictly positive. 

Our objective then is
\begin{equation}
\label{eq:optimization}
\begin{aligned}
& \underset{q(z|x,y)}{\text{minimize}} & & \mathcal{L}(q(z|x,y))\\
& \text{subject to}
& & \sum\limits_{z} q(z|x,y) = 1, \,\, \forall x,y, \\
& && q(z|x,y) \geq 0, \,\, \forall z, x, y.
\end{aligned}
\end{equation}

We can write the Lagrangian, which we represent with $\bar{\mathcal{L}}$, as
\begin{align}
\bar{\mathcal{L}}&=\sum\limits_{x,y,z} p(x,y)q(z|x,y)\log\left(\frac{q(z|x,y)}{q(z|x)q(z|y)}\right)+ I(X;Y)+( 1 - \beta) \sum\limits_{z}q(z)\log(q(z))\nonumber\\
&\hspace{0.3in}+ \sum\limits_{x,y}\delta_{x,y}\left(\sum\limits_{z} q(z|x,y)- 1\right)
\end{align}

In order to find the stationary points of the loss, we take its first derivative and set it to zero. To compute the partial derivatives, notice that $q(z|x),q(z|y),q(z)$ are linear functions of $q(z|x,y)$ (use Bayes rule and marginalization). We can then easily write the partial derivatives of these quantities with respect to $q(z|x,y)$ as follows: 
\begin{align*}
\frac{\del q(z|x)}{\del q(z|x,y)}&= \frac{\del{\sum\limits_{y'}q(z|x,y')p(y'|x) } }{\del{q(z|x,y)}}= p(y|x),\\
\frac{\del q(z|y)}{\del q(z|x,y)}&= \frac{\del{\sum\limits_{x'}q(z|x',y)p(x'|y)}}{\del{q(z|x,y)}} = p(x|y)\\
\frac{\del q(z)}{\del q(z|x,y)}&= \frac{\del{\sum\limits_{x',y'}q(z|x',y')p(x',y') }}{\del{q(z|x,y)}}= p(x,y).
\end{align*}

Using these expressions we have the following.
\begin{align*}
\frac{\del{\bar{\mathcal{L}}}}{\del{q(z|x,y)}} &= p(x,y)\left[1+\log(q(z|x,y)) \right.- (1+\log(q(z|x))) \nonumber\\
&\hspace{0.1in}- (1+\log(q(z|y)))\left.\hspace{0.1in} + (1-\beta)(1+\log(q(z))) + \delta_{x,y} \right]\\
&= p(x,y)\left[ - \beta + \delta_{x,y} 
+ \log\left(\frac{q(z|x,y)q(z)^{1-\beta}}{q(z|x)q(z|y)} \right) \right] 
\end{align*}
Assuming $p(x,y)>0$, any stationary point then satisfies
\begin{equation}
q(z|x,y) = \left(\frac{1}{2}\right)^{\delta_{x,y}-\beta}\frac{q(z|x)q(z|y)}{q(z)^{1-\beta}}
\end{equation}
Since $q(z|x,y)$ is a probability distribution, we have 
\begin{equation}
\sum\limits_{z}q(z|x,y) = \left(\frac{1}{2}\right)^{\delta_{x,y}-\beta}\sum\limits_{z} \frac{q(z|x)q(z|y)}{q(z)^{1-\beta}}= 1
\end{equation}
Defining $F(x,y)\coloneqq  \left(\frac{1}{2}\right)^{\delta_{x,y}-\beta}$, we have
\begin{equation}
F(x,y) = \frac{1}{\sum\limits_{z} \frac{q(z|x)q(z|y)}{q(z)^{1-\beta}}}.
\end{equation}

From the algorithm description, any stationary point of Algorithm \ref{iterative_algorithm} should satisfy 
\begin{equation}
q(z|x,y) = F(x,y) \frac{q(z|x)q(z|y)}{q(z)^{1-\beta}},
\end{equation}
for the same $F(x,y)$ defined above. Therefore a point is a stationary point of the loss function if and only if it is a stationary point of LatentSearch (Algorithm \ref{iterative_algorithm}). \hfill \qed

\subsection{Proof of Theorem \ref{thm:stationary}: Convergence}
\label{sec:alg_proof2}
In this section, we show the latter statement in Theorem \ref{thm:stationary}, i.e., \emph{LatentSearch} converges to a local minimum or a saddle point.
We can rewrite the loss as 
\begin{align}
\mathcal{L}(q(\cdot|\cdot,\cdot)) &= \sum\limits_{x,y,z} q(x,y,z)\log\left(\frac{q(x,y|z)}{q(x|z)q(y|z)}\right)- \beta \sum\limits_{z}q(z)\log(q(z))\\
&=  \sum\limits_{x,y,z} p(x,y)q(z|x,y)\log\left(\frac{q(z|x,y)}{q(z|x)q(z|y)}\right)  \nonumber\\ 
&\hspace{0.1in}+ I(X;Y)+( 1 - \beta) \sum\limits_{z}q(z)\log(q(z)),
\end{align}
If we substitute $\beta = 1$, we obtain
\begin{align}
\mathcal{L}(q(\cdot|\cdot,\cdot))  &=  \sum\limits_{x,y,z} p(x,y)q(z|x,y)\log\left(\frac{q(z|x,y)}{q(z|x)q(z|y)}\right) + I(X;Y).
\end{align} 

Our optimization problem can be written as
\begin{equation}
\label{eq:optimization2}
\begin{aligned}
& \underset{q(z|x,y)}{\text{minimize}} & & \mathcal{L}(q(z|x,y))\\
& \text{subject to}
& & \sum\limits_{z} q(z|x,y) = 1, \forall x,y.
\end{aligned}
\end{equation}
Notice that $\mathcal{L}(q(z|x,y))$ is not convex or concave in $q(z|x,y)$. However we can rewrite the minimization as follows:
\begin{equation}
\label{eq:optimization_augmented}
\begin{aligned}
& \underset{q(z|x,y)}{\text{minimize}} &  & \underset{r(z|x),s(z|y)}{\text{minimize}}   & & \sum\limits_{x,y,z} p(x,y)q(z|x,y)\nonumber\\
&&&&&\hspace{-0.5in}\log\left(\frac{q(z|x,y)}{r(z|x)s(z|y)}\right)  + I(X;Y)\\
&  & & \text{subject to}  & & \sum\limits_{z} q(z|x,y) = 1, \forall x,y  \\
& & & & &  \sum\limits_{z} r(z|x) = 1, \forall x,   \\
& & & & &  \sum\limits_{z} s(z|y) = 1, \forall y.
\end{aligned} 
\end{equation}
To see that (\ref{eq:optimization_augmented}) is equivalent to (\ref{eq:optimization2}), notice that the optimum for the inner minimization is $r^*(z|x) = q(z|x)$ and $s^*(z|x) = q(z|y)$. This is due to the fact that (\ref{eq:optimization_augmented}) is convex in $r(z|x)$ and $s(z|y)$ and concave in $t(z)$, which can be seen through the partial derivatives of the Lagrangian:
\begin{align}
&\underset{q(z|x,y)}{\text{min}} & & \underset{r(z|x),s(z|y)}{\text{min}}   & & \sum\limits_{x,y,z} p(x,y)q(z|x,y)+\log\left(\frac{q(z|x,y)}{r(z|x)s(z|y)}\right)  + I(X;Y) \\
&&&&& + \sum\limits_{x,y}\delta_{x,y}\left(\sum\limits_{z} q(z|x,y)- 1\right)+ \sum\limits_x \eta_x\left(\sum\limits_{z} r(z|x) - 1\right)\\
&&&&& + \sum\limits_x \nu_y\left(\sum\limits_{z} s(z|y) - 1\right) %_{\bar{\mathcal{L}}}
\end{align}
Let $\bar{\mathcal{L}}$ be defined as 
\begin{align}
\bar{\mathcal{L}}&=\sum\limits_{x,y}\delta_{x,y}\left(\sum\limits_{z} q(z|x,y)- 1\right)+ \sum\limits_x \eta_x\left(\sum\limits_{z} r(z|x) - 1\right) + \sum\limits_x \nu_y\left(\sum\limits_{z} s(z|y) - 1\right)
\end{align}

For fixed $q(z|x,y),s(z|y)$, we have
\begin{align*}
\frac{\del{\bar{\mathcal{L}}}}{\del{r(z|x)}} &= -\frac{p(x)q(z|x)}{r(z|x)}+\eta_x \\
\frac{\deltwo{\bar{\mathcal{L}}}}{\del{r(z|x)}^2} &=\frac{p(x)q(z|x)}{r(z|x)^2}.
\end{align*}
Therefore $\bar{\mathcal{L}}$ is convex in $r(z|x)$ and the optimum can be obtained by setting the first derivative to zero. Then we have
\begin{equation}
r^*(z|x) = \frac{p(x)q(z|x)}{\eta_x}, \forall x,z.
\end{equation}
Since we have $\sum\limits_{z}r^*(z|x) = \frac{p(x)}{\eta_x}\sum\limits_{z}q(z|x) = 1$, we obtain $r^*(z|x) = q(z|x)$. Similarly, we can show that $s^*(z|x) = q(z|y)$. Notice that this inner minimization is exactly the same as the first update of Algorithm \ref{iterative_algorithm}. 

We can also show that $\mathcal{L}$ is convex in the variables $r,s$ jointly: This can be seen through the fact that $\frac{\deltwo{}}{\del{r(z|x)s(z|y)}}\mathcal{L}=0$ and the Hessian is positive definite.

This concludes that (\ref{eq:optimization_augmented}) is equivalent to (\ref{eq:optimization}). Moreover, since the objective function is convex in $q(z|x,y)$ and also jointly convex in $r(z|x),s(z|y)$, we can switch the order of the minimization terms. Therefore, we can equivalently write
\begin{align}
\label{eq:optimization_augmented2}
& \underset{r(z|x),s(z|y)}{\text{min}}   & & \underset{q(z|x,y)}{\text{min}} & & \sum\limits_{x,y,z} p(x,y)q(z|x,y)\log\left(\frac{q(z|x,y)}{r(z|x)s(z|y)}\right)  + I(X;Y)+\bar{\mathcal{L}}
\end{align}

Let us analyze the inner minimization in this equivalent formulation for fixed $r(z|x),s(z|x)$. Similarly, we can take the partial derivative as follows:
\begin{align*}
\frac{\del{\bar{\mathcal{L}}}}{\del{q(z|x,y)}} &= p(x,y)\left[1+\log(q(z|x,y)) - \log(r(z|x)) - \log(s(z|x)) +  \delta_{x,y} \right]\\
&= p(x,y)\left[ 1 + \delta_{x,y} + \log\left(\frac{q(z|x,y)}{r(z|x)s(z|y)} \right) \right] \\
\frac{\deltwo{\bar{\mathcal{L}}}}{\del{q(z|x,y)}^2} &= p(x,y)\left[ \frac{1}{q(z|x,y)}  \right].
\end{align*}

Notice that $\frac{\deltwo{\bar{\mathcal{L}}}}{\del{q(z|x,y)}^2} > 0$. Hence $\bar{\mathcal{L}}$ is convex in $ q(z|x,y)$. Then the optimum can be obtained by setting the first derivative to zero. We have
\begin{equation}
p(x,y)\left[ 1 + \delta_{x,y} + \log\left(\frac{q(z|x,y)}{r(z|x)s(z|y)} \right) \right] = 0,
\end{equation}
or equivalently
\begin{equation}
q(z|x,y) = \left(\frac{1}{2}\right)^{1+\delta_{x,y}}r(z|x)s(z|y).
\end{equation}
Note that if we define 
$$
F(x,y)\coloneqq \sum\limits_{z}r(z|x)s(z|y),
$$ 
since 
$\sum\limits_{z}q(z|x,y) = \left(\frac{1}{2}\right)^{1+\delta_{x,y}}\sum\limits_{z}r(z|x)s(z|y) = 1$, we can write 
\begin{equation}
q(z|x,y) = \frac{1}{F(x,y)}r(z|x)s(z|y).
\end{equation}
This is exactly the same as the second update of LatentSearch (Algorithm \ref{iterative_algorithm}) if $r(z|x)=q(z|x), s(z|y)=q(z|y)$. 

Therefore, if $q_i(z|x,y)$ is the current conditional at iteration $i$, the next update of LatentSearch (Algorithm \ref{iterative_algorithm}) is equivalent to first solving the inner minimization of (\ref{eq:optimization_augmented}) thereby assigning $r(z|x) = q_i(z|x), s(z|y)=q_i(z|y)$, then switching the order of the minimization operations, and solving the inner minimization of (\ref{eq:optimization_augmented2}), therefore assigning $q_{i+1}(z|x,y) = \frac{1}{F(x,y)}q_i(z|x)q_i(z|y)$. In each of this two-step optimization iteration, either loss function goes down, or it does not change. If it does not change, the algorithm has converged. Otherwise, it cannot go down indefinitely since loss (\ref{eq:loss}) is lower bounded as $I(X;Y|Z)\geq 0$ and $H(Z)\geq 0$ and therefore has to converge. This proves convergence of the algorithm to either a local minimum or a saddle point. The converged point cannot be a local maximum since it is arrived at after a minimization step. \hfill \qed

\subsection{Proof of Theorem \ref{thm:identifiability}}
We first show the result for distinguishing latent graph from the triangle graph.

Since $X,Y$ are discrete variables, we can represent the joint distribution of $X,Y$ in matrix form. Let $\mat{M} = [p(x,y)]_{(x,y)\in [m]\times[n]}$. With a slight abuse of notation, let $\mat{z}\coloneqq \left[z_1, z_2,\hdots z_k \right]$ be the probability mass (row) vector of variable $Z$, i.e., $\p{Z=i} = \mat{z}[i] = z_i$. Similarly, let $\mat{x}_z\coloneqq \left[ x_{z,1}, x_{z,2}, \hdots x_{z,k}\right]$ be the conditional probability mass vector of $X$ conditioned on $Z=z$, i.e., $\p{X=i|Z=z} = \mat{x}_z[i] = x_{z,i}$. Finally, let $\mat{y}_{z,x} \coloneqq \left[ y_{z,x,1}, y_{z,x,2}, \hdots y_{z,x,n} \right]$ be the conditional probability mass vector of $Y$ conditioned on $X=x$ and $Z=z$. We can write the matrix $\mat{M}$ as follows:
\begin{equation}
\label{eq:M}
\mat{M} = \sum\limits_{i=1}^k z_i \begin{bmatrix}
x_{i,1}\mat{y}_{i,1} \\
x_{i,2}\mat{y}_{i,2} \\
\vdots \\
x_{i,m}\mat{y}_{i,m}
\end{bmatrix}
\end{equation}

Now suppose for the sake of contradiction that there exists such a $q(x,y,z)$ such that $\sum_{z}q(x,y,z) = p(x,y)$ and $\ci{X}{Y}{Z}$. Then $\mat{M}$ admits a factorization of the form 
\begin{equation}
\mat{M} = \sum\limits_{i=1}^k z_i' \begin{bmatrix}
x_{i,1}'\mat{y}_{i,1}' \\
x_{i,2}'\mat{y}_{i,2}'\\
\vdots \\
x_{i,m}'\mat{y}_{i,m}'
\end{bmatrix},
\end{equation}
where $x_{i,j}',\mat{y}_{i,j}',z_i'$ are due to the joint $q(x,y,z)$ and are potentially different form their counterparts in (\ref{eq:M}). Notice that since $\ci{X}{Y}{Z}$, we have $\mat{y}_{i,j}' = \mat{y}_{i,l}', \forall (j,l)\in [k]\times[m]$. Therefore the matrices
\begin{equation}
\begin{bmatrix}
x_{i,1}'\mat{y}_{i,1}' \\
x_{i,2}'\mat{y}_{i,2}'\\
\vdots \\
x_{i,m}'\mat{y}_{i,m}'
\end{bmatrix},
\end{equation}
are rank 1 $\forall i\in [k]$. Therefore, $\mat{M}$ has NMF rank at most $k$. Since matrix rank is upper bounded by the NMF rank, $\text{rank}(M)\leq k$. Therefore, there exists a $q(x,y,z)$ such that $\sum\limits_{z}q(x,y,z) = p(x,y)$ and $\ci{X}{Y}{Z}$ \emph{only if} $\text{rank}(M)\leq k$. In fact, it is easy to show that this is an if and only if relation: Any NMF of the joint distribution corresponds to a latent confounder and and latent confounder corresponds to an NMF of the joint distribution. Next, we show that under the generative model described in the theorem statement, this happens with probability zero.

We have the following lemma:
\begin{lemma}
	\label{lem:simplex}
	Let $\{\mat{x}_i: i \in [n]\}$ be a set of vectors sampled independently, uniformly randomly from the simplex $S_{n-1}$ in $n$ dimensions. Then, $\{\mat{x}_i: i \in [n]\}$ are linearly independent with probability 1.
\end{lemma}
\begin{proof}
	If $\mat{x}_i$ are linearly dependent, then there exists a set $\{\alpha_i : i\in[n]\}$ such that $\sum\limits_{i=1}^n\alpha_i\mat{x}_i=0$. Let $j = \argmax\{i\in [n]:\alpha_i>0 \}$. Equivalently $\mat{x}_j$ is in the range of the set of vectors $\{\mat{x}_i : i\in [j-1]\}$. Therefore, we can write 
	\begin{align}
	\p{\{\mat{x}_i: i \in [n]\} \text{ are linearly independent }}\nonumber\\
	&\hspace{-1.8in}\leq \sum\limits_{i=2}^n \p{\mat{x}_i\in R(\mat{x}_1,\hdots,\mat{x}_{i-1})},\label{eq:terence}
	\end{align} 
	where $R(\mat{x}_1,\hdots,\mat{x}_{i-1})$, is the range of the vectors $\mat{x}_1,\hdots,\mat{x}_{i-1}$, i.e., the vector space spanned by $\mat{x}_1,\hdots,\mat{x}_{i-1}$. 
	
	Notice that $\text{dim($ R(\mat{x}_1,\hdots,\mat{x}_{i-1}) $)}< n-1, \forall i\leq n-1 $. Therefore, codimension of $R(\mat{x}_1,\hdots,\mat{x}_{i-1}) $ with respect to the simplex is non-zero $\forall i \leq n-1$. Therefore, the Lebesgue measure of $R(\mat{x}_1,\hdots,\mat{x}_{i-1}) \cap S_{n-1}$ is zero with respect to the uniform measure over $S_{n-1}$. Hence, $\p{\mat{x}_i\in R(\mat{x}_1,\hdots,\mat{x}_{i-1})} = 0, \forall i\leq n-1$.
	
	The above argument does not hold for the last term in the summation in (\ref{eq:terence}). However, intersection of any $n-1$ dimensional vector space with the simplex $S_{n-1}$ is an $n-2$ dimensional slice of the simplex \cite{webb1996central}. Therefore, it has Lebesgue measure zero with respect to the uniform measure over the simplex.
\end{proof}

\begin{corollary}
	\label{cor:simplex}
	Let $\{\mat{x}_i: i \in [n]\}$ be a set of vectors sampled independently, uniformly randomly from the simplex $S_{n-1}$ in $n$ dimensions. Let $\{c_i \neq 0: i\in[n] \}$ be arbitrary real scalars that are non-zero. Then, $\{c_i\mat{x}_i: i \in [n]\}$ are linearly independent with probability 1.
\end{corollary}
\begin{proof}
	The proof of Lemma \ref{lem:simplex} goes through since the span of a set of vectors does not change with scaling of the vectors.
\end{proof}

$\mat{M}$ is rank deficient if and only if its determinant is zero, i.e.,  $\text{det}(M) = 0$. The determinant is a polynomial in $\{z_i:i\in [k]\}$. By induction, one can show that if a finite degree multivariate polynomial is not identically zero, the set of roots has zero Lebesgue measure (for example, see \cite{Caron2015}). The uniform measure over the simplex is absolutely continuous with respect to Lebesgue measure. Hence, the set of roots of a finite degree multivariate polynomial has measure zero with respect to the uniform measure over the simplex.

To show that $\text{det}(\mat{M})$ is not identically zero, it is sufficient to choose a set of $z_i's$ for which determinant is non-zero. First, observe that by Corollary  \ref{cor:simplex}, each matrix 
\begin{equation}
\begin{bmatrix}
x_{i,1}\mat{y}_{i,1} \\
x_{i,2}\mat{y}_{i,2} \\
\vdots \\
x_{i,m}\mat{y}_{i,m}
\end{bmatrix}
\end{equation}
is full rank with probability 1. Let $z_1 = 1$ and $z_j , \forall j\in \{2,3,\hdots, k\}$. Then $\text{det}(\mat{M})\neq 0$ since $\mat{M}$ is full rank. Therefore, the determinant, which is a polynomial in $\{z_i : i\in[k]\}$ is not identically zero. This concludes the proof that with probability 1, $\text{rank}(\mat{M})=n>k$. 

If the distribution is generated from the direct graph, from Lemma \ref{lem:simplex}, we know that the rows of conditional probability matrix are linearly independent. Since joint probability matrix can be obtained by scaling each row of this matrix with the probability values of $X$, and this operation does not change rank, joint probability matrix obtained from the direct graph is full rank with probability 1. Therefore non-negative rank of this matrix has to be $n$, concluding the proof.
\hfill \qed

\subsection{Proof of Corollary \ref{cor:cardinality}}
The statement follows from the fact that the proposed generative model induces a non-zero probability measure on every joint distribution, which is the set of distributions that can be encoded from the \emph{triangle graph} and any distribution that can be encoded by the \emph{latent graph} requires $\ci{X}{Y}{Z}$, which we show in Theorem \ref{thm:identifiability} happens with probability zero.\hfill \qed

\subsection{Proof of Theorem \ref{thm:binary_identifiability}}
We give the proof for binary $Z$. The argument can be extended to when $Z$ has any finite number of states.

We overload the notation and use $z$ for the probability that random variable Z is $0$. We have
\begin{align}
&p(Z=0) = z,\hspace{0.35in} p(Z=1) = 1-z\\
&p(X=0|Z=z)=x_z\\
&p(Y=0|X=x,Z=z)=y_{x,z}
\end{align}

The conditional distributions from UGM can be sampled uniformly from the simplex via normalized exponential random variables, however in the case of binary variables, this is equivalent to sampling uniformly. Hence, we can assume $x_z,y_{x,z}$ are uniform random variables with support $[0,1]$. Based on this generative model, we can calculate $p(x)$ and $p(y|x)$ as follows:
\begin{align}
&p(X=0) = x_0z+x_1(1-z),\label{eq:generative1}\\
&p(X=1) = (1-x_0)z + (1-x_1)(1-z)\nonumber\\
&p(Y=0|X=0) = \frac{y_{0,0}x_0z + y_{0,1}x_1(1-z)}{x_0z+x_1(1-z)}\label{eq:generative2}\\
&p(Y=0|X=1) = \frac{y_{1,0}(1-x_0)z + y_{1,1}(1-x_1)(1-z)}{(1-x_0)z + (1-x_1)(1-z)}\label{eq:generative3}
\end{align}

We use the characterization of \cite{kumar2014exact} for the minimum entropy $Z$ that can make $X,Y$ conditionally independent. Let $t = p(X=0)$ and let $\alpha\coloneqq p(Y=0|X=0), \beta\coloneqq p(Y=0|X=1)$. We re-state their theorem for self-containment of our paper:
\begin{theorem}[\cite{kumar2014exact}]
\label{thm:LB}
Consider two binary random variables $X,Y$. Define $t\coloneqq p(X=0),\alpha\coloneqq p(Y=0|X=0),  \beta\coloneqq p(Y=0|X=1)$. Let $\alpha' = \min\{\alpha,\beta\},\beta'\coloneqq\max\{\alpha,\beta\}$. Then of all $q(x,y,z')$ where $q(x,y)=p(x,y)$ and $\ci{X}{Y}{Z}$ minimum entropy $Z'$ has entropy
\begin{align}
&LB\coloneqq \min\{ H_b(A),H_b(B)\},\\
&A = t\left(1-\frac{\alpha'}{\beta'}\right), B = (1-t)\left(1-\frac{1-\beta'}{1-\alpha'}\right) \label{eq:common_entropy_lower_bound}
\end{align} 
\end{theorem}

Note that in the generative model we are considering, the entries of $p(x,y)$ are random variables, which implies that $LB$ is a random variable. 

Consider a sequence $z_n$. Let $Z_n$ be the binary random variable where $\mathbb{P}(Z_n=0)=z_n$. Notice that $H(z_n)$ converges to zero if and only if $z_n$ converges to either $0$ or $1$. Since the generative model is symmetric with respect to the conditionals $p(x|z=0)$ compared to $p(x|z=1)$ and $p(y|x,z=0)$ compared to $p(y|x,z=1)$, without loss of generality we can consider the case where $z_n$ goes to $0$.

Now suppose $0<z_0<0.5$ and $z_n$ is a monotonically decreasing sequence. When we substitute $z_n$ for $z$ in the generative model, we use the symbols in Theorem \ref{thm:LB} with subscript $n$ to distinguish them for different values of $n$.

The event that there does not exist a latent variable with small entropy that can make the observed variables independent is equivalent to the event that the lower bound is strictly greater than the entropy of the true latent variable:
\begin{align}
 \mathbb{P}(\mathcal{Q}_p=\emptyset)
&=\mathbb{P}(p(x,y): \nexists q(x,y,z')\\
& \text{ s.t. } \sum_{z'} q(x,y,z')=p(x,y), \\
&\hspace{0.35in} \ci{X}{Y}{Z'}, H(Z')\leq H(Z)) \\
&= \mathbb{P}(LB> H(Z))
\end{align}

We want to show that
\begin{equation}
\lim_{n\rightarrow \infty} \mathbb{P}(LB_n > H(Z_n)) = 1.
\end{equation}

Define the following events:
\begin{align}
\varepsilon_{z_n}^A \coloneqq \{\text{Event that } H_b(A_n)\leq H_b(z_n)\}.\\
\varepsilon_{z_n}^B \coloneqq \{\text{Event that } H_b(B_n)\leq H_b(z_n)\}.
\end{align}
By union bound
\begin{align}
\mathbb{P}(LB_n \leq H(Z_n)) &\leq \mathbb{P}(\varepsilon_{z_n}^A)+\mathbb{P}(\varepsilon_{z_n}^B)
\end{align}
We first investigate the term $\lim_{n\rightarrow \infty}\mathbb{P}(\varepsilon_z^A)$. By conditioning on the event that $A_n\leq0.5$ and $A_n>0.5$, we can reduce the comparison between $H_b(a),H_b(c)$ to a comparision between $a$ and $c$. Due to first applying the law of total probability and then Bayes rule, we have
\begin{align}
\mathbb{P}(\varepsilon_{z_n}^A) &=\mathbb{P}\left(\varepsilon_{z_n}^A\vert A_n\leq 0.5\right)\mathbb{P}(A_n\leq 0.5) + \mathbb{P}\left(\varepsilon_{z_n}^A\vert A_n>0.5\right)\mathbb{P}(A_n>0.5)\nonumber\\
&=\mathbb{P}\left(A_n \leq z_n \vert A_n\leq 0.5\right)\mathbb{P}(A_n\leq 0.5) + \mathbb{P}\left(A_n \geq 1 - z_n\vert A_n>0.5\right)\mathbb{P}(A_n>0.5)\nonumber\\
&=\mathbb{P}\left(A_n \leq z_n \right)\mathbb{P}(A_n\leq 0.5\vert A_n \leq z_n ) + \mathbb{P}\left(A_n\geq 1 - z_n\right)\mathbb{P}(A_n>0.5\vert A_n \geq 1 - z_n)\nonumber\\
&=\mathbb{P}\left(A_n \leq z_n \right)+ \mathbb{P}\left(A_n\geq 1 - z_n\right)\nonumber
\end{align}
Define the following random variable:
\begin{eqnarray}
S_n^A\coloneqq -z_n + t_n\left(1-\frac{\alpha_n'}{\beta_n'}\right),
\end{eqnarray}
where the terms on the right hand side are as defined in Theorem \ref{thm:LB}.  Then $\pp{A_n \leq z_n} = \pp{S_n^A\leq 0}$ and $\pp{A_n \geq 1-z_n}=\pp{S_n^A\geq 1}$. We have
\begin{align}
\pp{S_n^A\leq 0} = \int_{-\infty}^0 S_n^A d\mu,
\end{align}
where $\mu$ is the probability measure induced by the generative model. Note that $t_n\in[0,1],z_n\in(0,0.5),\frac{\alpha_n'}{\beta_n'}\in (0,1]$, we have $|S_n|\leq 1$. Then from the dominated convergence theorem since $\int 1 d\mu = 1 <\infty$, we have
\begin{align}
\lim_{n\rightarrow \infty} \int_{-\infty}^0 S_n^A d\mu &= \int_{-\infty}^0 \lim_{n\rightarrow \infty} S_n^Ad\mu.
\end{align}
We have 
\begin{align}
&\lim_{n\rightarrow \infty}S_n^A = \lim_{n\rightarrow \infty} -z_n + t_n\left(1-\frac{\alpha_n'}{\beta_n'}\right)
\end{align}
Both $\alpha_n$ and $\beta_n$ are random variables supported on $[0,1]$. Moreover, since limit exists for $\alpha_n,\beta_n$, it also exists for $\alpha_n'\coloneqq \min\{\alpha_n,\beta_n\}$, similarly it exists for $\beta_n'$. Therefore,
\begin{align}
&\lim_{n\rightarrow \infty} -z_n + t_n(1-\frac{\alpha_n'}{\beta_n'}) = x_1\left(1 -\frac{\lim_n \alpha_n'}{\lim_n \beta_n'} \right)\\
&=x_1\left(1 - \frac{\lim_n \min\{\alpha_n,\beta_n\}}{\lim_n \max\{\alpha_n,\beta_n\}}\right)\\
&=x_1\left(1 - \frac{ \min\{\lim_n\alpha_n,\lim_n\beta_n\}}{ \max\{\lim_n\alpha_n,\lim_n\beta_n\}}\right)\\
& = x_1\left(1-\frac{\min\{y_{0,1},y_{1,1}\}}{\max\{y_{0,1},y_{1,1}\}}\right)
\end{align}
where the last equation follows from the equations (\ref{eq:generative1})-(\ref{eq:generative3}). Finally, we have that 
\begin{align}
&\int_{-\infty}^0 x_1\left(1-\frac{\min\{y_{0,1},y_{1,1}\}}{\max\{y_{0,1},y_{1,1}\}}\right) d\mu \\
&= \pp{x_1\left(1-\frac{\min\{y_{0,1},y_{1,1}\}}{\max\{y_{0,1},y_{1,1}\}}\right)\leq 0}\\
& = \pp{x_1\left(1-\frac{\min\{y_{0,1},y_{1,1}\}}{\max\{y_{0,1},y_{1,1}\}}\right)= 0}=0\label{eq:limprob}\\
\end{align}
where the last two equations follow from the fact that $x_1\left(1-\frac{\min\{y_{0,1},y_{1,1}\}}{\max\{y_{0,1},y_{1,1}\}}\right)$ is supported in the interval $[0,1]$ and has an  absolutely continuous measure, which implies that measure of a single point is zero.

Similarly, we can calculate 
\begin{eqnarray}
\pp{S_n^A\geq 1}  = \int_1^\infty S_n^Ad\mu = 0,
\end{eqnarray}
which leads to $\pp{\varepsilon_{z_n}^A}=0$. 

Next, we consider the same analysis for $\pp{\varepsilon_{z_n}^B}$. One difference is the replacement of $t_n$ with $1-t_n$ which does not affect the derivation except for the replacement of $x_1$ with $1-x_1$. Moreover, in the numerator within the paranthesis in $(\ref{eq:limprob}$, $\min\{y_0,1,y_{1,1}\}$ is replaced with $1-\max\{y_0,1,y_{1,1}\}$ and similarly in the denominator $\max\{y_0,1,y_{1,1}\}$ is replaced with $1-\min\{y_0,1,y_{1,1}\}$. It follows that $\pp{\varepsilon_{z_n}^B}=0$. This implies that $\lim_{n\rightarrow \infty}\pp{LB_n\leq H(Z_n)}= 0$ concluding the proof.
\hfill\qed

\subsubsection{Comments on Distinguishing Direct Graph from Latent Graph with Entropy}
\label{sec:direct_vs_latent_discussion}
Consider the uniform generative model for the triangle graph. It is easy to see that in this case, in $(\ref{eq:common_entropy_lower_bound})$, $t,\alpha',\beta'$ become independent and uniformly distributed random variables with the compact support $[0,1]$. One can calculate the distribution of this lower bound accordingly. This can be used to obtain the probability of identifiability between direct graph and the latent graph for a given upper bound on the entropy of the latent variable. We do not pursue this calculation here.

\subsubsection{Extension to $Z$ with $k$ states}
Consider the setting where $Z$ has $k$ states.  We use the following notation in this section:
\begin{equation}
p(Z=i) = z^{(i)}.
\end{equation}
First, note that the characterization of \cite{kumar2014exact} is still applicable since they show increasing the dimension of $Z$ to more than two states cannot reduce the minimum entropy. Similar to the above proof, we will assume a sequence of random variables $Z_n$. Let $Z_n$ be a sequence of random variables with the pmf 
\begin{equation}
p(Z_n=i) = z_n^{(i)}.
\end{equation}
Note that $H(Z_n)\rightarrow 0$ if and only if $\exists i \in [k]$ such that $z_n^{(i)}\rightarrow 1$ and $(z_n^{(j)})_{j\neq i}\rightarrow \mat{0}$. In the following, we show that a similar analysis to the binary case goes through irrrespective of how $(z_n^{(j)})_{j\neq i}$ converges to the zero vector. Suppose without loss of generality $z_n^{(1)}\rightarrow 1$.

Due to the grouping  rule of entropy, we have
\begin{equation}
H(Z_n) = H_b(z_n^{(1)})+(1-z_n^{(1)})H((w_n^{i})_{2\leq i\leq k}),
\end{equation}
where $w_n^{(i)} = \frac{z_n^{(i)}}{1-z_n^{(1)}}$. Let $N$ be such that $z_N^{(1)}\geq 1-\frac{\epsilon}{\log_2(k)}$. Then $z_n^{(1)}\geq 1-\frac{\epsilon}{\log_2(k)} ,\forall n>N$. Then we have $H(Z_n) \leq H_b(z_n^{(1)})+\epsilon,\forall n\geq N$.

Now we can replicate the proof for the binary $Z$ as follows. Let us define the events:
\begin{align*}
\varepsilon_{n}^A &\coloneqq \{\text{Event that } H_b(A_n)\leq H_b(z_n^{(1)})\}.\\
\varepsilon_{n}^B &\coloneqq \{\text{Event that } H_b(B_n)\leq H_b(z_n^{(1)})\}.\\
\delta_{n}^A &\coloneqq \{\text{Event that } H_b(z_n^{(1)})<H_b(A_n)\leq H(Z_n)\}.\\
\delta_{n}^B &\coloneqq \{\text{Event that } H_b(z_n^{(1)})<H_b(B_n)\leq H(Z_n)\}.
\end{align*}
By union bound
\begin{align}
&\mathbb{P}(LB_n \leq H(Z_n)) \leq \mathbb{P}(\varepsilon_{n}^A)+\mathbb{P}(\varepsilon_{n}^B)+\mathbb{P}(\delta_{n}^A)+\mathbb{P}(\delta_{n}^B)
\end{align}
We can write 
\begin{equation}
\mathbb{P}(\delta_n^A) \leq \int_{H_b(z_n^{(1)})}^{H_b(z_n^{(1)})+\epsilon}H_b(A)d\mu.
\end{equation}
It is easy to see that $\lim_{n\rightarrow \infty}\mathbb{P}(\delta_n^{A})=0$ since $\epsilon\rightarrow 0$ and the measaure induced by the generative model is absolutely continuous in the integral.

The rest of the analysis follows similarly to the binary case: We can obtain expressions for $\alpha,\beta$ using $k$ terms instead of 2. In the limit, all but the term that contains $z_n^{(1)}$ go to zero and we can conclude the proof using the same arguments. \hfill \qed
\begin{figure}[t!]
	\begin{center}
		\includegraphics[width=0.4\textwidth]{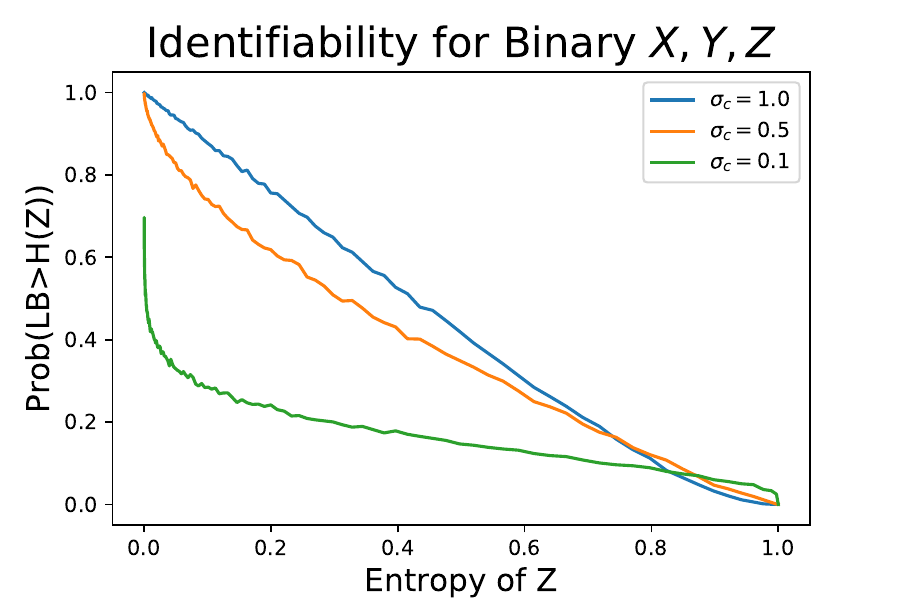}
	\end{center}
	\caption{The measure of distributions from the triangle graph that does not fit to latent graph with a latent at least as simple as the true latent. As the entropy of the true latent goes to $0$, this fraction goes to $1$. This is precisely the measure of models which are identifiable with our entropic approach in the case of binary $X,Y,Z$. }
	\label{fig:binary_identifiability}
\end{figure}
\begin{figure}[t!]
	\vskip 0.2in
	\begin{center}
		\includegraphics[width=0.55\columnwidth]{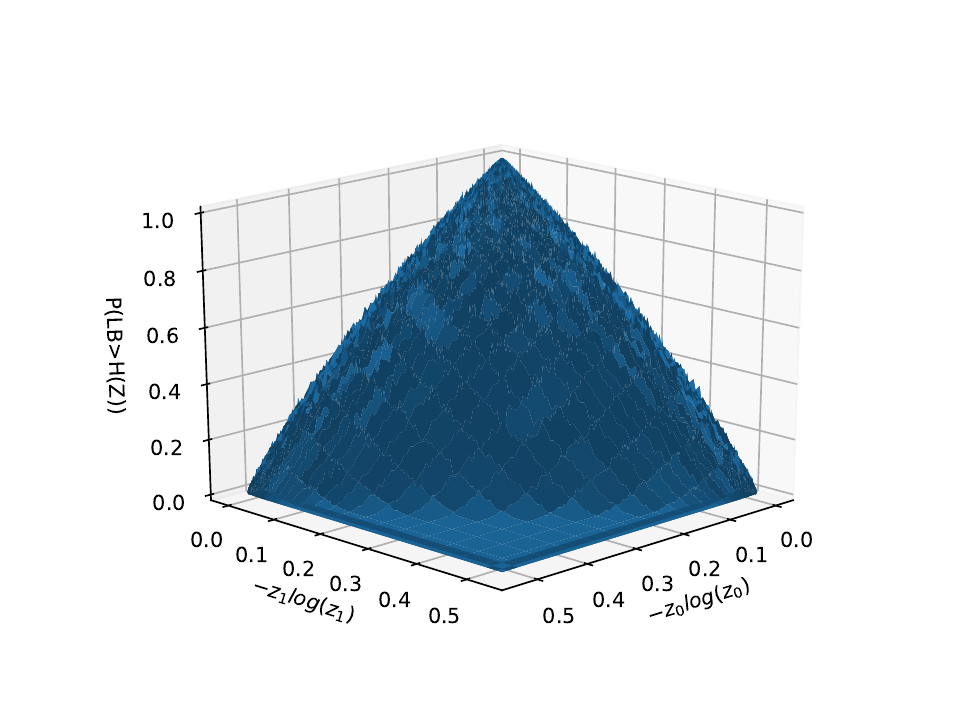}\caption{Fraction of distributions from the triangle graph that does not fit to latent graph with a latent at least as simple as the true latent for binary $X,Y$ and ternary $Z$. As the entropy of the true latent goes to $0$, this fraction goes to 1. This is precisely the fraction of models which are identifiable with our entropic approach in the case of binary $X,Y,Z$. }
		\label{fig:ternary_z_identifiability}
	\end{center}
\end{figure}
\subsection{Entropic Identifiability for Binary Variables}
\label{app:binary}
For $H(Z)>0$, we can approximate the fraction of identifiable causal models via simulations. We sample probability distributions from the uniform generative model. For each sample we check if the entropy of the true latent $H(Z)$ is strictly less than the common entropy of observed variables $G_1(X,Y)$. These are the models from triangle graph which cannot be fit onto latent graph with a low-entropy latent. Figure \ref{fig:binary_identifiability} shows this fraction. $\sigma_c$ is the parameter of the Dirichlet distribution used to sample the conditional distributions from the Triangle graph. $\sigma_c=1$ corresponds to uniform sampling model and smaller the $\sigma_c$ more deterministic the conditional distributions are in the sampling model. 

See Figure \ref{fig:ternary_z_identifiability} for the corresponding plot for ternary $Z$.

\subsection{$I(X;Y|Z)$ vs. $H(Z)$ Tradeoff Curve}
Figure \ref{fig:example} shows the $I(X;Y|Z)$-$H(Z)$ tradeoff LatentSearch (Algorithm \ref{iterative_algorithm}) obtains for a joint distribution sampled as follows: The distribution of $Z$ as well as the conditional distributions $p(X|z),p(Y|z), \forall z$ are chosen uniformly at random over the simplex.
\begin{figure*}[t!]
\vskip 0.2in
\begin{center}
\label{fig:example1}
\includegraphics[width=0.35\textwidth]{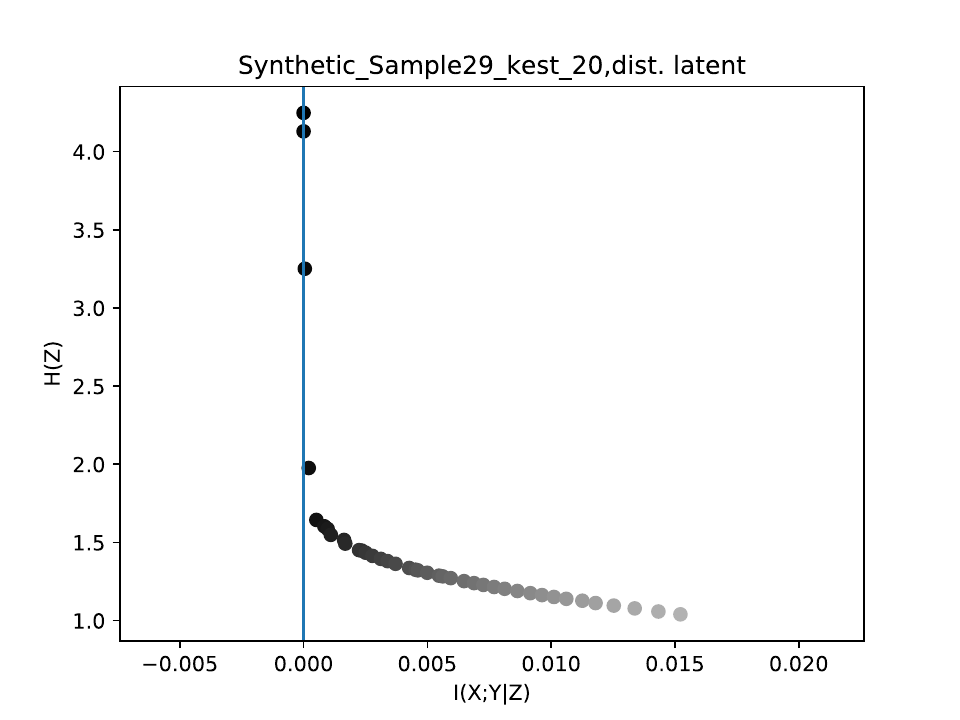}\caption{$I(X;Y|Z)$ vs. $H(Z)$ tradeoff curve obtained by LatentSearch (Algorithm \ref{iterative_algorithm}) for an arbitrary joint $p(x,y)$ from the graph $X\leftarrow Z \rightarrow Y$. We observed that the curve's shape is consistent across many runs irrespective of the graph, although the crossing point where $I(X;Y|Z)=0$ changes.}
\label{fig:example}
\end{center}
\end{figure*}

\subsection{Complete Details of Experiments}
\label{app:sim_details}
In this section, we explain some of the key implementation details for the experiments in Section \ref{sec:experiments} that were left out from the main text due to space constraints. 

\textbf{Sampling a low-entropy latent: } In many experiments, we sample distributions from either the latent graph or the triangle graph such that entropy of the latent confounder $Z$ is small. For example, we enforce $H(Z)\leq \theta$ for varying thresholds $\theta$ in Figure \ref{fig:synthetic_identifiability} . We use a form of rejection sampling combined with sampling from Dirichlet distributions with low-entropy as follows:

Suppose, we need $N$ samples where $H(Z)\leq \theta$. We initialize $\alpha^{(0)}=1$ and obtain $10N$ samples from $Dir(\alpha^{(0)})$. If we have at least $N$ samples where $H(Z)\leq \theta$, we are done. If not, we update $\alpha$ by halving it, i.e., $\alpha^{(1)}=0.5\alpha^{(0)}$. The lower the $\alpha$ value, the lower-entropy distributions we will obtain from $Dir(\alpha)$. Then we repeat this process until iteration $i$, where at least $N$ samples can be obtained from $10N$ samples using $Dir(\alpha^{(i)})$. We conclude by analyzing the histogram plots of $H(Z)$ that this method not only allows us to sample distributions where $H(Z)\leq \theta$ but also where $H(Z)\approx \theta$, providing us with a better control over the entropy of the latent confounder. 

\textbf{Choosing number of states of latent variable in LatentSearch: } Recall that LatentSearch allows us to discover a tradeoff between $H(Z)$ and $I(X;Y|Z)$, which, combined with a $I(X;Y|Z)$ for conditional independence, can be used to approximate common entropy. Since only $X,Y$ are observed, we do not know how many states $k$ $Z$ has. As pointed out in the main text, one can try all $k\leq mn$ without loss of generality, where $m,n$ are the number of states of $X,Y$, respectively. However, in practice, this takes a long time. Furthermore, we identified that this is not necessary for the estimation of common entropy.

We observe that if we search over $Z$ with very large number of states, e.g., $k=mn$, performance of LatentSearch does not improve compared to having $k=\min\{m,n\}$. This is because the number of optimization parameters increases significantly which may require many more iterations. It also slows down the algorithm. We observed that choosing $k=\max\{m,n\}$ provides the smallest entropy latents in practice. Therefore, we set $k=\max\{m,n\}$ in LatentSearch for our experiments. 

\textbf{Sampling DAGs for testing EntropicPC in Figure \ref{fig:entropicpc_synthetic}: } We first sample Erd\"{o}s-\renyi graphs with parameter $0.2$. Since there are $10$ nodes, this corresponds to an average degree of $2$ per node. Note that these graphs are undirected. We need to make them directed and ensure there are no cycles. For this, we randomly picked a total order between the nodes and directed the edges respecting that total order. It can be easily shown that the resulting graphs have no cycles. 

\textbf{Sampling joint distributions for a given DAG in Figure \ref{fig:entropicpc_synthetic}: } For every DAG we generate, we need to obtain a joint distribution from which we sample a dataset. To obtain a distribution for a given graph, we employ a method from Chickering and Meek~\cite{chickering2002finding}. It is known that constraint-based methods require the data to be faithful to the graph, i.e., every pair of variables that are connected in the graph should be dependent. This notion should also be true under any conditioning set. In practice, this does not always hold. Specifically, nodes that are far away from each other in the graph might be almost statistically independent. To ensure faithfulness in practice, Chickering and Meek use a method to sample conditional distributions for a given DAG~\cite{chickering2002finding}. In summary, it ensures that parent-child relations are far from uniform. The details of their sampling method, which we also use, are as follows:

For a variable $X$ with $m$ states, they define the vector $\mat{v}=\frac{1}{T}(\frac{1}{1},\frac{1}{2},\hdots, \frac{1}{m}), T=\sum_{i=1}^m\frac{1}{i}$. For the $j^{th}$ instantiation of the parent set of $X$, $pa_x^{(j)}$, they use the vector $\mat{v}_j$ which is the $j$-shifted version of $\mat{v}$. The shifts are cyclic in the sense that $\mat{v}_m=\mat{v}$. They later sample $p(X|pa_x^{(j)})$ from the Dirichlet distribution with parameter vector $\mat{v}_j$. When each coordinate parameter of Dirichlet is identical, the expected distribution is uniform. When Dirichlet is sampled with parameters $\mat{v}_j$ as given above, each coordinate has a different parameter. Indeed, the expected distribution becomes $\mat{v}_j$ rather than uniform. Therefore, this method ensures that $p(x|pa_x^{j})$ is typically far from uniform, encouraging strong dependence between parents and the children in the graph.  

\textbf{Details specific to Figure \ref{fig:latentsearch_performance}: }
We sampled $100$ distributions from the latent graph for each value of $n$, where $X,Y,Z$ all have $n$ states. In each distribution, we ensure that $H(Z)\leq 1$ using the low-entropy sampling method described above. We use LatentSearch on $50$ different values of $\beta$, uniformly spaced in the interval $[0,0.1]$. We set the latent variable for LatentSearch to $n$ number of states. We run LatentSearch for $500$ iterations each time. We used the conditional mutual information threshold of $0.001$: In other words, of the algorithm outputs for the $50$ $\beta$ values used, we pick the smallest entropy $Z$ discovered by the algorithm among those that ensure $I(X;Y|Z)\leq 0.001$. We then compare this value with the entropy of the true latent confounder.

\textbf{Details specific to Figure \ref{fig:synthetic_conjecture}: }
We sample $1000$ distributions from the triangle graph. As mentioned in the main text, we use LatentSearch output to approximate common entropy. The settings for LatentSearch are the same as above, i.e., we use $50$ $\beta$ values uniformly spaced in the range $[0,0.1]$, use $n$ states for the latent and run the algorithm for $500$ iterations. Entropy recovered by LatentSearch for a pair $X,Y$ is then compared with $\min\{H(X),H(Y)\}$. The $y-$axis shows that fraction of times the reconstructed latent has entropy of at least $\alpha\min\{H(X),H(Y)$ for different values of $\alpha$.

\textbf{Details specific to Figure \ref{fig:synthetic_identifiability}: }
For this figure we sample $1000$ distributions from both triangle graph and the latent graph for various upper bounds on the entropy of $Z$. Low-entropy sampling is done as explained before. Finally, Algorithm \ref{causal_inference_algorithm} is used to identify the true causal graph with $\theta=0.8\min\{H(X),H(Y)\}$. LatentSearch settings used within Algorithm \ref{causal_inference_algorithm} are as given previously.  

\textbf{Details specific to Figure \ref{fig:tuebingen}: }
Tuebingen dataset consists of around $100$ real cause-effect pairs. We run LatentSearch to understand whether real cause-effect pairs can be made independent by low-entropy variables. As explained in the main text, we used different conditional independence thresholds. Visual inspection of $I-H$ curves suggest that $0.001$ is a good threshold for this dataset. This can be done by checking, for the given range of $\beta$ values, where the curve disengages form the $x=0$ axis. We used $100$ $\beta$ values in the range $[0,0.1]$ and run the LatentSearch algorithm for $1000$ iterations for this experiment.

\textbf{Details specific to Figure \ref{fig:IHplane}: }
This figure is an example of the tradeoff curve LatentSearch discovers for various values of $\beta$. Each dot corresponds to a joint distribution $p(x,y,z)$ constructed by LatentSearch for a given value of $\beta$ after a certain number of iterations. As can be seen from (\ref{eq:loss}), smaller $\beta$ values enforce smaller $I(X;Y|Z)$. The horizontal line indicates $\min\{H(X),H(Y)\}$. $X,Y$ can always be separated with this much entropy since by definition $\ci{X}{Y}{X}$ and $\ci{X}{Y}{Y}$. Ideally, i.e., with infinite samples and infinitely many $
\beta$ values, the point that intersects the $x=0$ line (i.e., the $y-$ axis) should give the common entropy. To account for finite-sample effects, we use a different horizontal line, which we call conditional mutual information threshold, as described before. 

\textbf{Details specific to Figure \ref{fig:mediator_synthetic}: }
We sample from the graph $X\rightarrow Z\rightarrow Y$ and investigate $H(Z)$. Note that $Z$ acts as a mediator if it is not observed. Our goal is to understand if it is typical to have low-entropy mediators. We set the dimensions of $X,Y,Z$ to $n$. If $Z$ has $k$ states, $H(Z)\leq \log(k)$. Our goal is to demonstrate that unless $k$ is a constant, $H(Z)$ scales similar to $H(X),H(Y)$. Most of the details of this experiment are provided in the main text. Note that when $\alpha_{Dir}\leq \frac{1}{n}$ for a  distribution with $n$ states, a sample from Dirichlet distribution typically looks very peaked, i.e., it has very high probability for one of the states, and very low probabilities for the rest. When such a low $\alpha_{Dir}$ is used to sample the conditional of $p(Z|x)$ for every value of $x$, $X$ and $Z$ are almost deterministically related, i.e., there is almost no additional entropy introduced in the system. We show that even then the entropy of the mediator scales. Larger $\alpha_{Dir}$ values will give distributions that are closer to uniform, which in turn will make $Z$ close to uniform and have $\log(n)$ entropy.

\textbf{Details specific to Figure \ref{fig:entropicpc_synthetic}: }
Our goal is to demonstrate how well output of EntropicPC approximates a true causal graph by checking graphical distances between the skeleton and essential graph. Essential graph is the mixed graph where undirected edges show the edges that cannot be learned whereas directed edges show the edges that can be learned. Structural Hamming Distance counts the number of edges that should be reverted, added or removed to change the output graph into the true graph. For skeleton discovery, we see edge discovery as a classification problem and calculate the F-score as an established summary of the classifier performance. The graph and distribution sampling are described above. We sample a dataset with $100000$ variables and for each figure, we subsample varying number of samples from this dataset without replacement. This is repeated for $100$ different graphs, and their corresponding distributions. 

\subsection{EntropicPC on Line and Collider Graphs}
\label{sec:appEntropicPC}
To demonstrate effectiveness of EntropicPC compared to PC on the simplest possible graph, we conducted the experiments of Figure \ref{fig:entropicpc_synthetic} on the line graph $X\rightarrow Y \rightarrow Z$ and the collider graph $X\rightarrow Y\leftarrow Z$. The results are given in Figures \ref{fig:entropicpc_line} and \ref{fig:entropicpc_collider}, respectively.
\begin{figure}[t]
	\centering
	\begin{subfigure}[b]{0.3\textwidth}
		\includegraphics[width=\textwidth]{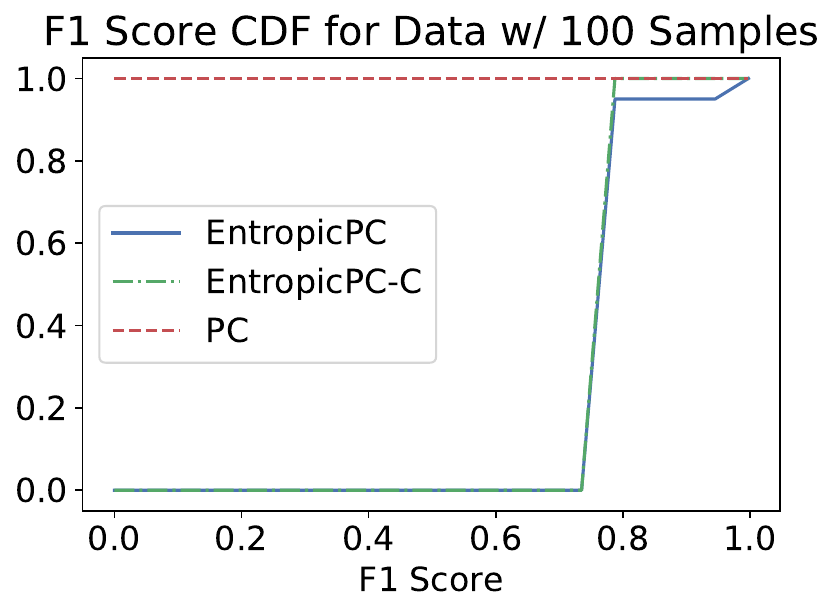}
		\caption{ }
		\label{fig:F1_n3N100_cum_line}
	\end{subfigure}
	\begin{subfigure}[b]{0.3\textwidth}
		\includegraphics[width=\textwidth]{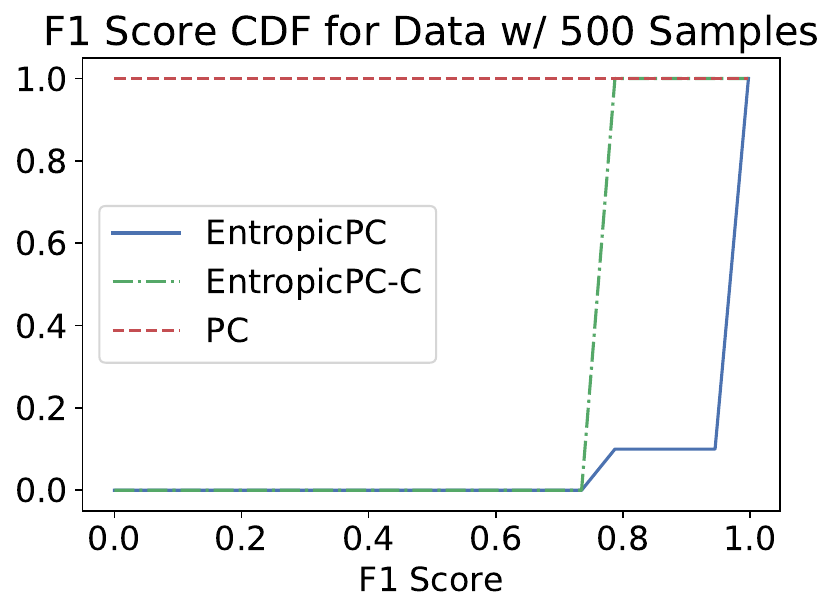}
		\caption{ }
		\label{fig:F1_n3N500_cum_line}
	\end{subfigure}
	\begin{subfigure}[b]{0.3\textwidth}
		\includegraphics[width=\textwidth]{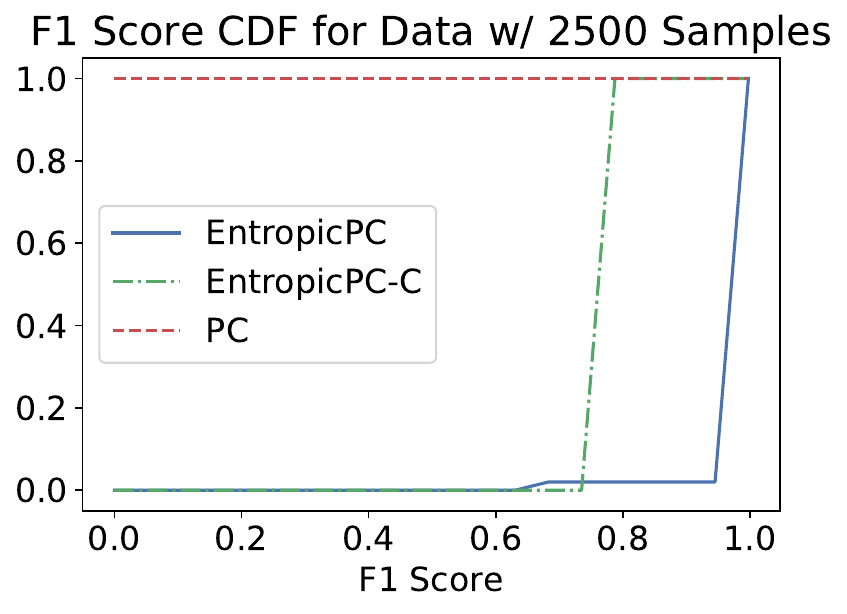}
		\caption{ }
		\label{fig:F1_n3N2500_cum_line}
	\end{subfigure}
	\begin{subfigure}[b]{0.3\textwidth}
		\includegraphics[width=\textwidth]{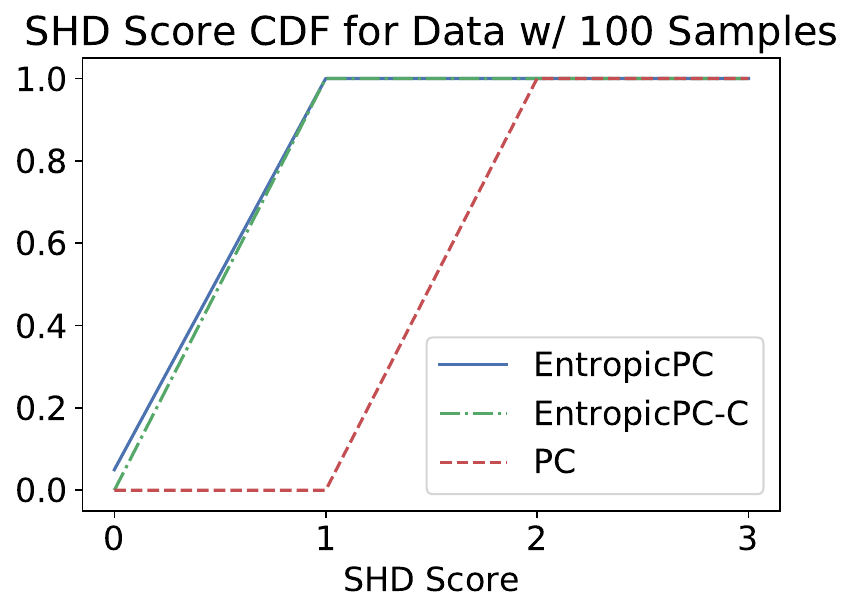}
		\caption{ }
		\label{fig:SHD_n3N100_cum_line}
	\end{subfigure}
	\begin{subfigure}[b]{0.3\textwidth}
		\includegraphics[width=\textwidth]{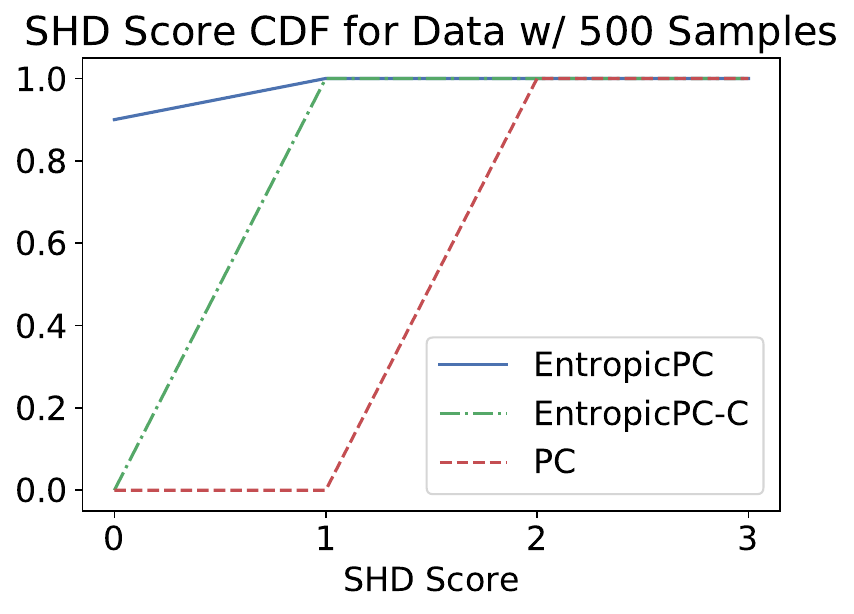}
		\caption{ }
		\label{fig:SHD_n3N500_cum_line}
	\end{subfigure}
	\begin{subfigure}[b]{0.3\textwidth}
		\includegraphics[width=\textwidth]{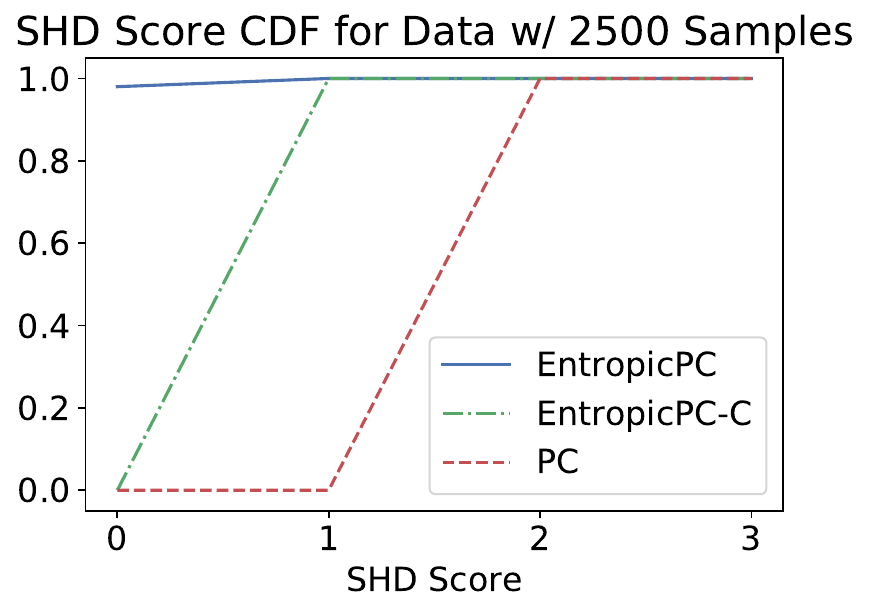}
		\caption{ }
		\label{fig:SHD_n3N2500_cum_line}
	\end{subfigure}
	\caption{Performance of EntropicPC, EntropicPC-C and PC in the graph $X\rightarrow Y\rightarrow Z$. All three methods perform identically for $10000$ samples (not shown). Both EntropicPC and EntropicPC-C consistently outperforms baseline PC. }
	\label{fig:entropicpc_line}
\end{figure}

\begin{figure}[t]
	\centering
	\begin{subfigure}[b]{0.3\textwidth}
		\includegraphics[width=\textwidth]{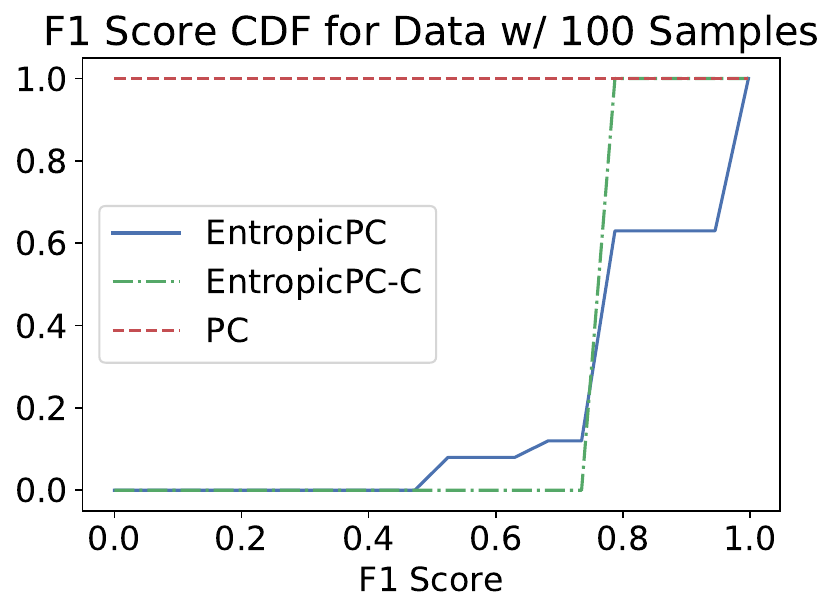}
		\caption{ }
		\label{fig:F1_n3N100_cum_collider}
	\end{subfigure}
	\begin{subfigure}[b]{0.3\textwidth}
		\includegraphics[width=\textwidth]{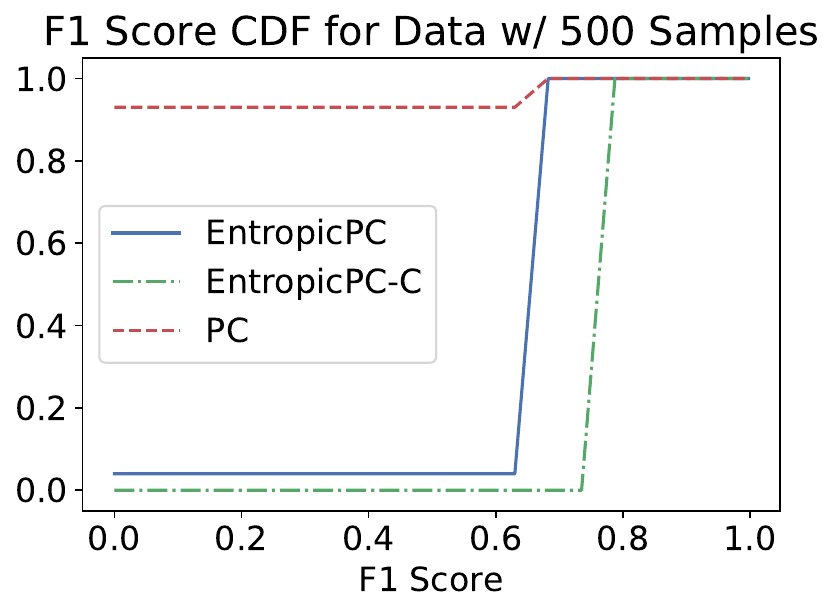}
		\caption{ }
		\label{fig:F1_n3N500_cum_collider}
	\end{subfigure}
	\begin{subfigure}[b]{0.3\textwidth}
		\includegraphics[width=\textwidth]{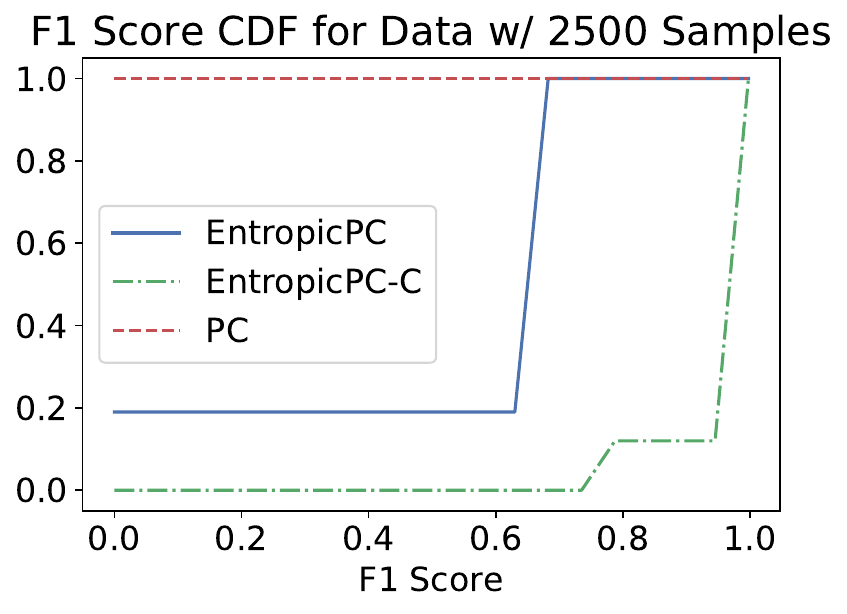}
		\caption{ }
		\label{fig:F1_n3N2500_cum_collider}
	\end{subfigure}
	\begin{subfigure}[b]{0.3\textwidth}
		\includegraphics[width=\textwidth]{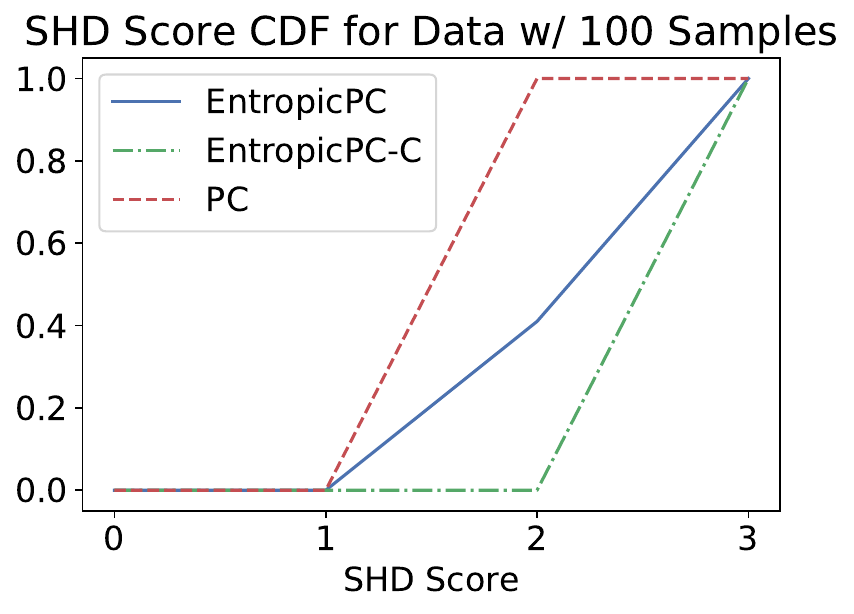}
		\caption{ }
		\label{fig:SHD_n3N100_cum_collider}
	\end{subfigure}
	\begin{subfigure}[b]{0.3\textwidth}
		\includegraphics[width=\textwidth]{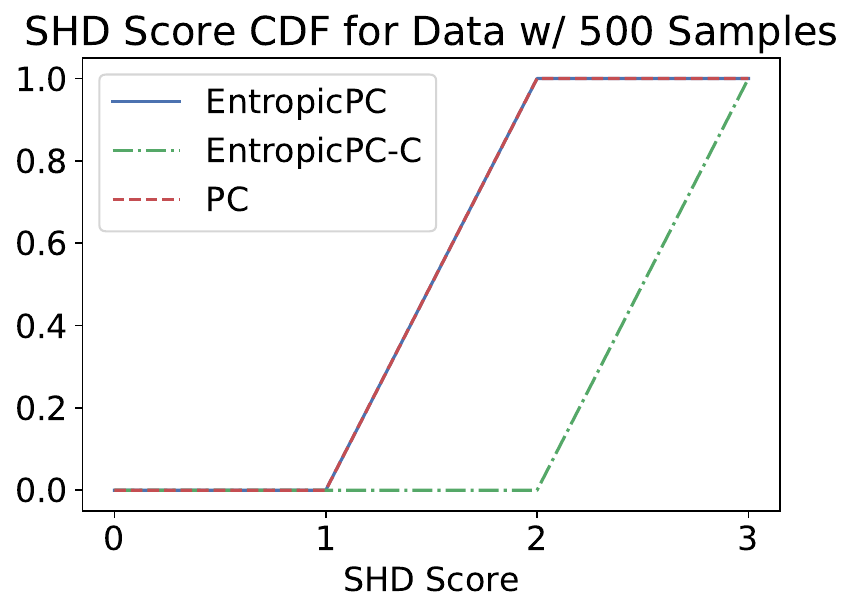}
		\caption{ }
		\label{fig:SHD_n3N500_cum_collider}
	\end{subfigure}
	\begin{subfigure}[b]{0.3\textwidth}
		\includegraphics[width=\textwidth]{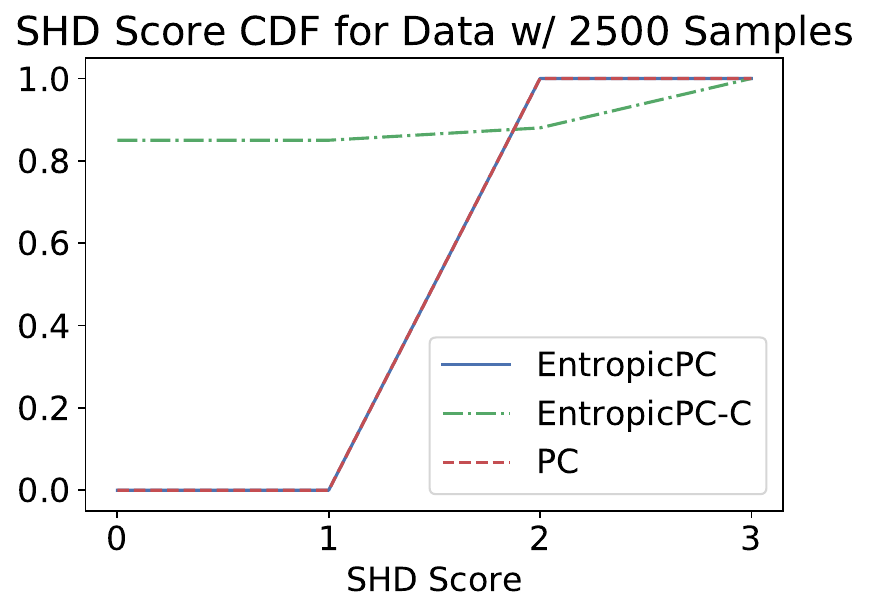}
		\caption{ }
		\label{fig:SHD_n3N2500_cum_collider}
	\end{subfigure}
	\caption{Performance of EntropicPC, EntropicPC-C and PC in the graph $X\rightarrow Y\leftarrow Z$. All three methods perform identically for $10000$ samples (not shown). Different from the line graph, for collider graph in the very small sample regime, even though the proposed methods significantly outperforms PC in terms of skeleton discovery, PC gets a slight edge in performance in terms of SHD. Note that PC tends to declare variables independent in the small sample regime. It is likely that for $100$ samples, PC often estimates empty graph, which has SHD of $2$ from the essential graph. This explains the discrepancy between the performance for F1 and SHD scores for $100$ samples.}
	\label{fig:entropicpc_collider}
\end{figure}
 
\subsection{Comparing LatentSearch with EM, NMF and Gradient descent}
\label{sec:EM_NMF_GD}

\subsubsection{Comparison to gradient descent}
Instead of using LatentSearch for minimizing the loss in $(\ref{eq:loss})$, one can use gradient descent. Even though the objective is not convex, gradient descend will still output a stationary point if it converges. However gradient descend comes with many practical issues, as we detail in the following, and support with experiments.

First, we observed that iterative update step is slightly faster than the gradient descent step: Average time for iterative update:0.000063 seconds. Average time for gradient update: 0.000078 seconds. More importantly, gradient descent takes much longer to converge and does not even achieve the same performance.

As observed in Figure \ref{fig:simulation2}, gradient descent converges only after 350000 iterations, whereas we observed that iterative update converges after around 200 iterations. Based on the average update times, this corresponds to a staggering difference of 0.01 seconds for the iterative algorithm vs.\ 27.3 seconds for the gradient descent algorithm. Although these results are for when $n=m=k=5$ states, we observed single iterative update to be faster than single gradient update, giving similar performance comparison results for $n=m=k=80$ states. 

The above result is for a constant step size of $0.001$. With smaller step size, convergence slows down even further. With larger step size, gradient descent does not converge. 
\begin{figure*}[t!]
\vskip 0.2in
\begin{center}
    \begin{subfigure}[b]{0.32\textwidth}
\label{fig:iter0200}\includegraphics[width=\textwidth]{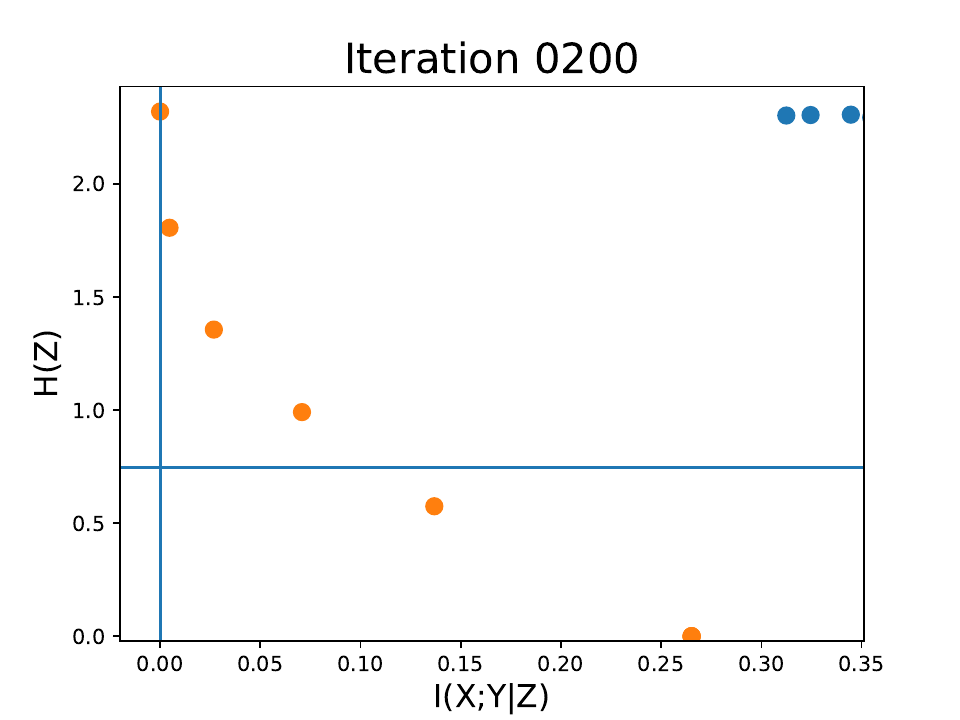}
	\caption{ }
\end{subfigure}	
    \begin{subfigure}[b]{0.32\textwidth}
	\label{fig:iter16000}\includegraphics[width=\textwidth]{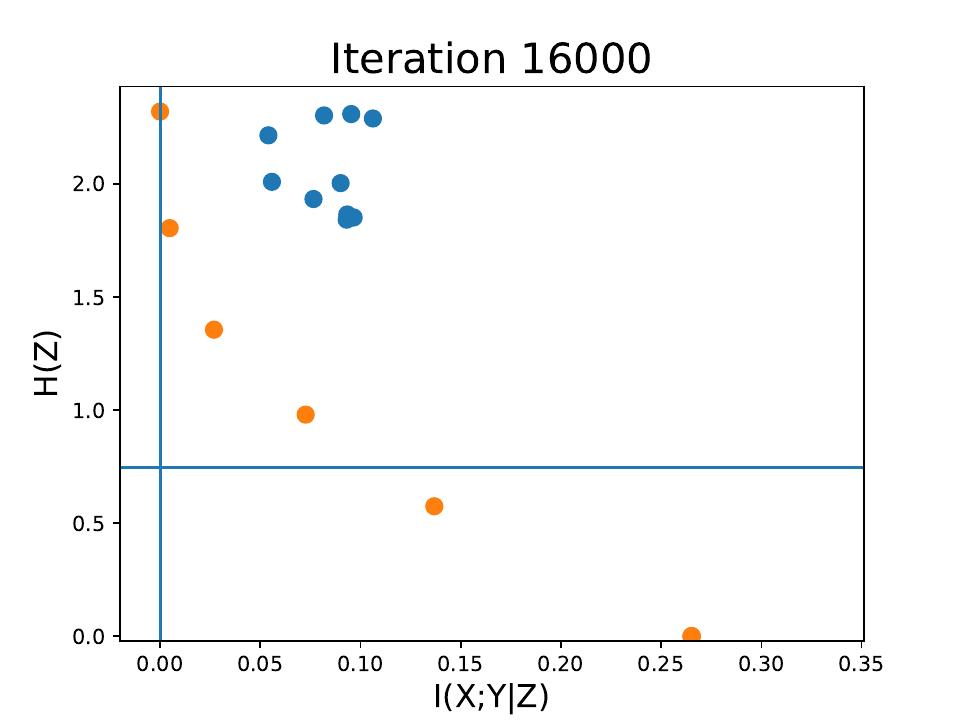}
	\caption{ }
\end{subfigure}	
    \begin{subfigure}[b]{0.32\textwidth}
	\label{fig:iter350000}\includegraphics[width=\textwidth]{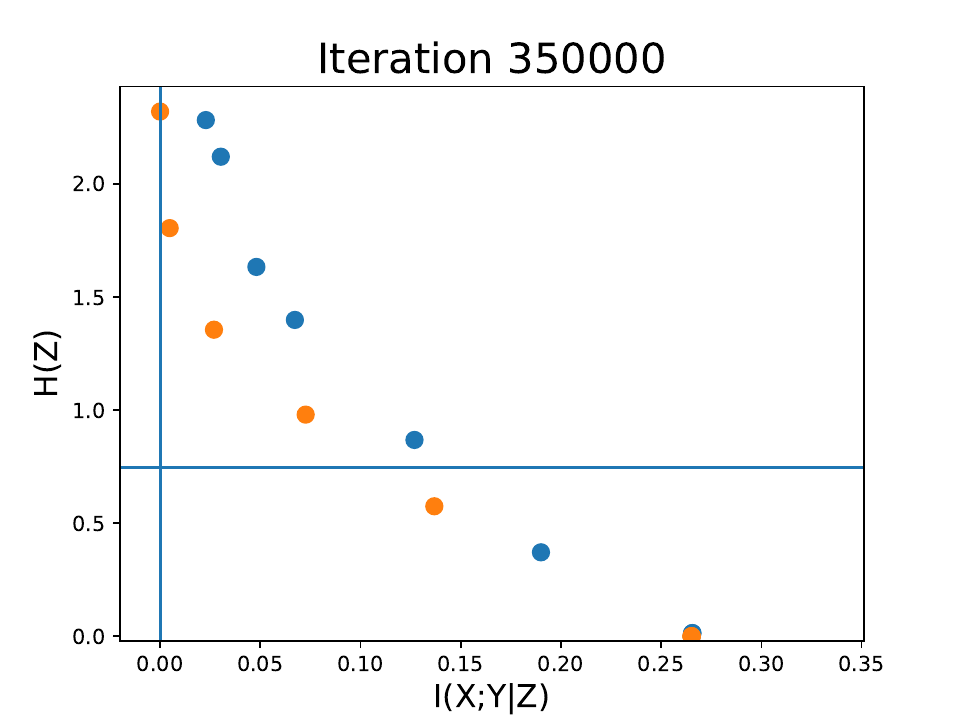}
	\caption{ }
\end{subfigure}	
\caption{Comparison of the iterative algorithm with gradient descent. Blue points show the trajectory of gradient descent, whereas orange points show the trajectory for Algorithm \ref{iterative_algorithm} for 10 randomly initialized points with different $\beta$ values in loss (\ref{eq:loss}). Gradient descent takes 350,000 iterations to converge whereas iterative algorithm converges in about 200 iterations. Moreover, the points achieved by iterative algorithm are strictly better than gradient descent after convergence.  }
\label{fig:simulation2}
\end{center}
\vskip -0.2in
\end{figure*}

\subsubsection{Comparison with EM algorithm}
EM is the first algorithm suggested for solving the pLSA problem \cite{hofmann1999probabilistic}. For the details of EM within this framework, please see \cite{hofmann1999probabilistic}. However the EM algorithm for pLSA problem does not have any incentive to minimize the entropy of the latent factor. 

In order to see how EM affects the entropy of the discovered latent variable, we run EM algorithm by initializing it at the points that are output by LatentSearch (Algorithm \ref{iterative_algorithm}). Results are illustrated in Figure \ref{fig:simulation3}. We observe that the points obtained in the $I-H$ plane migrate towards $I(X;Y|Z)=0$ line, while staying above what we believe is a fundamental lower bound curve. This step does not improve our algorithm, as it increases the entropy of the latent variables. 

\begin{figure*}[ht!]
\vskip 0.2in
\begin{center}
	    \begin{subfigure}[b]{0.32\textwidth}
	    \label{fig:EM0}
		\includegraphics[width=\textwidth]{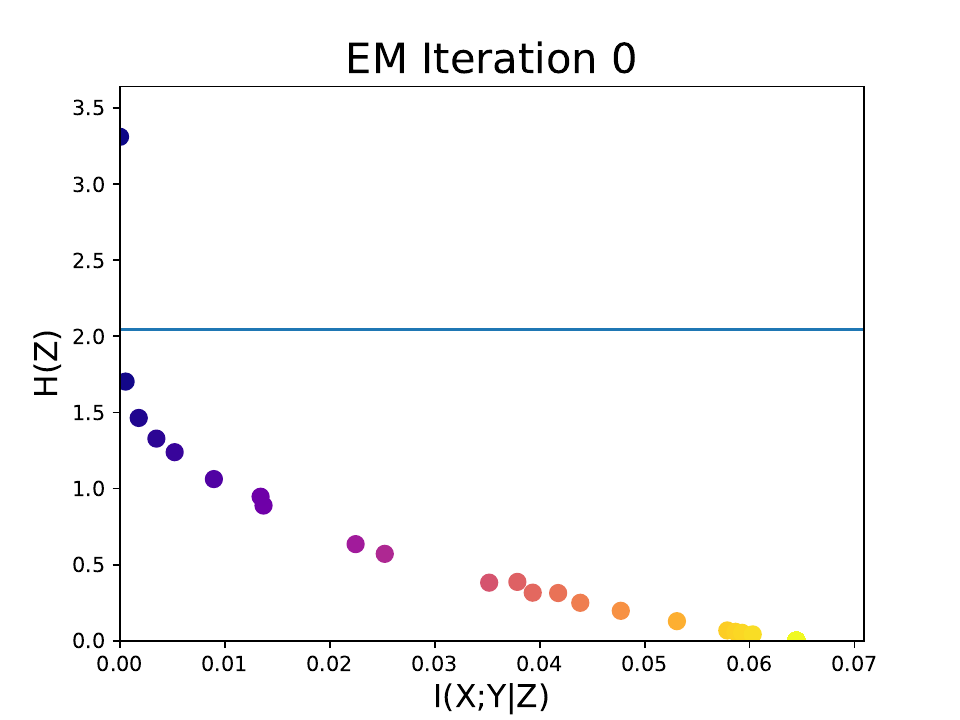}
		\caption{ }
	\end{subfigure}
	    \begin{subfigure}[b]{0.32\textwidth}
    \label{fig:EM60}
	\includegraphics[width=\textwidth]{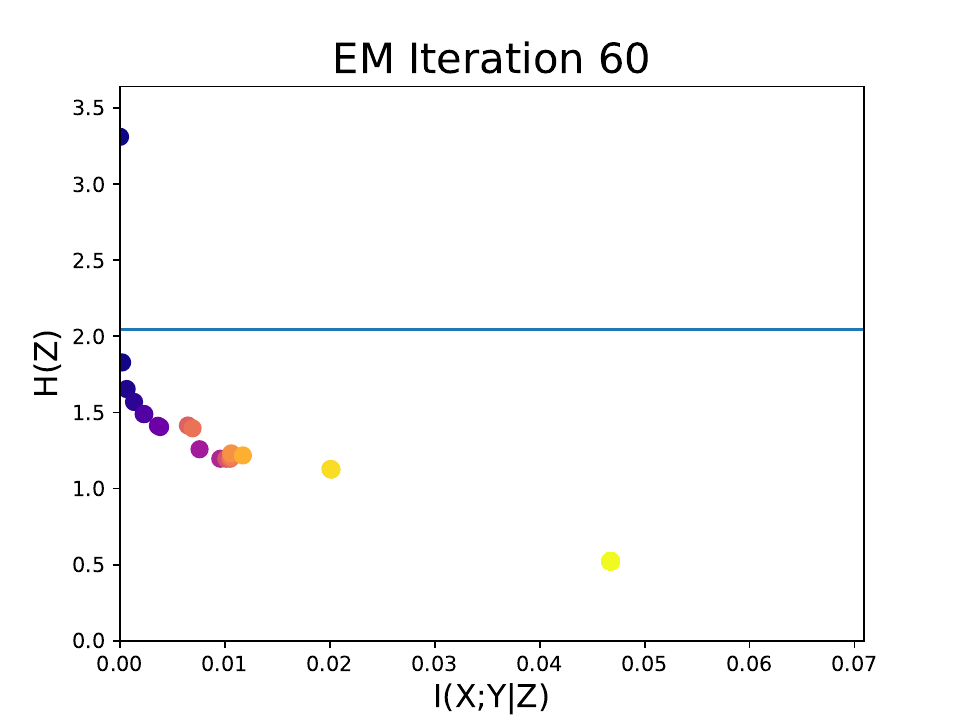}
	\caption{ }
\end{subfigure}
	    \begin{subfigure}[b]{0.32\textwidth}
	\label{fig:EM300}
	\includegraphics[width=\textwidth]{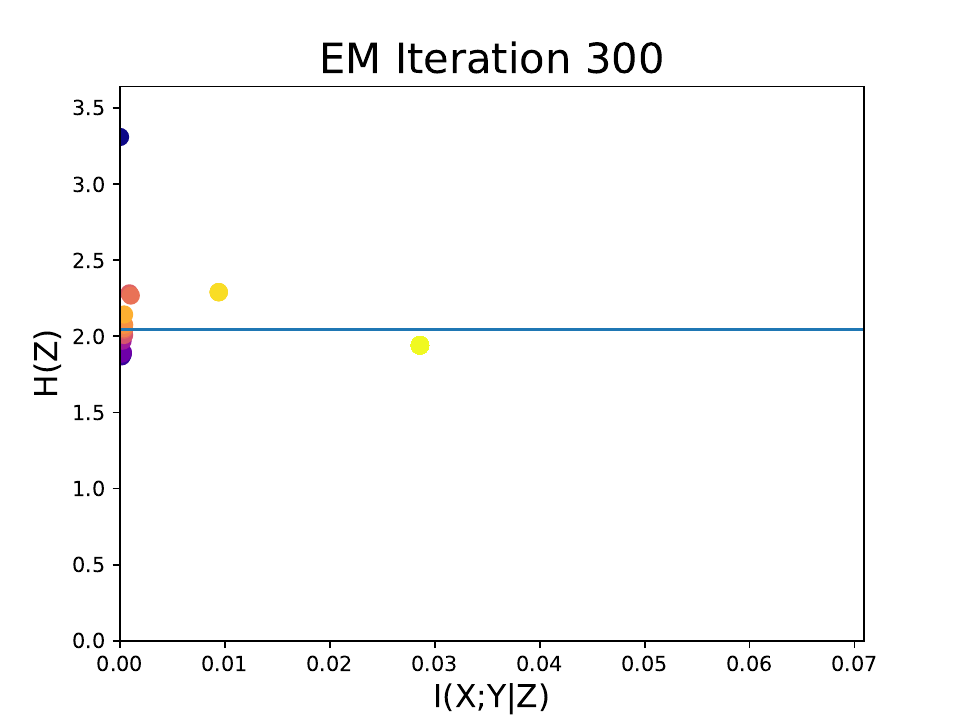}
	\caption{ }
\end{subfigure}
\caption{Applying EM to the output of iterative algorithm migrates points to $I(X;Y|Z=0)$ line: (a) Latent variables discovered by LatentSearch (Algorithm \ref{iterative_algorithm}) shown on the $I(X;Y|Z)-H(Z)$ plane. (b,c) After applying EM algorithm on the points in (a) after 60 and 300 iterations. Observe that the points always remain above the line depicted by LatentSearch (Algorithm \ref{iterative_algorithm}).}
\label{fig:simulation3}
\end{center}
\vskip -0.2in
\end{figure*}

\subsubsection{Comparison with NMF}
Consider the joint distribution matrix $\mat{M}$. Suppose we find an approximation to this matrix as $\mat{M}\approx \mat{UV}$ where the common dimension of $\mat{U,V}$ is $k$ through NMF. This is equivalent to setting the dimension of the latent variable to $k$. This can be seen as a hard entropy threshold on the entropy of the latent factor since $H(Z)\leq \log(k)$. We can sweep through different dimensions and see how NMF performs compared to LatentSearch (Algorithm \ref{iterative_algorithm}). Note that NMF is in general hard to solve. A commonly used approach is the iterative algorithm: Initialize $\mat{U}_0,\mat{V}_0$. Find the best $\mat{U}_1$ such that $\mat{M} \approx \mat{U}_1\mat{V}_0$. Then find the best $\mat{V}_1$ such that $\mat{M} \approx \mat{U}_1\mat{V}_1$ and iterate. In the experiments, we used this iterative algorithm together with $l_1$ loss.

\begin{figure*}[h!]
\vskip 0.2in
\begin{center}
    \begin{subfigure}[b]{0.4\textwidth}
    \label{fig:nmf_latent}
	\includegraphics[width=\textwidth]{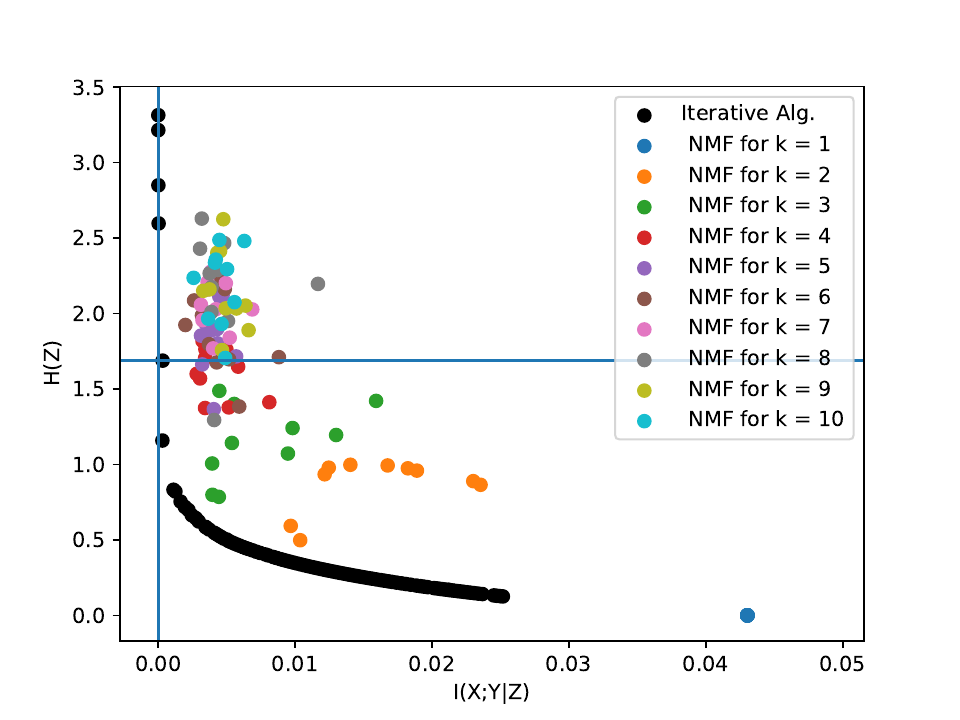}
	\caption{Causal Graph $X\leftarrow Z \rightarrow Y$}
\end{subfigure}	
    \begin{subfigure}[b]{0.4\textwidth}
	\label{fig:nmf_complete}
	\includegraphics[width=\textwidth]{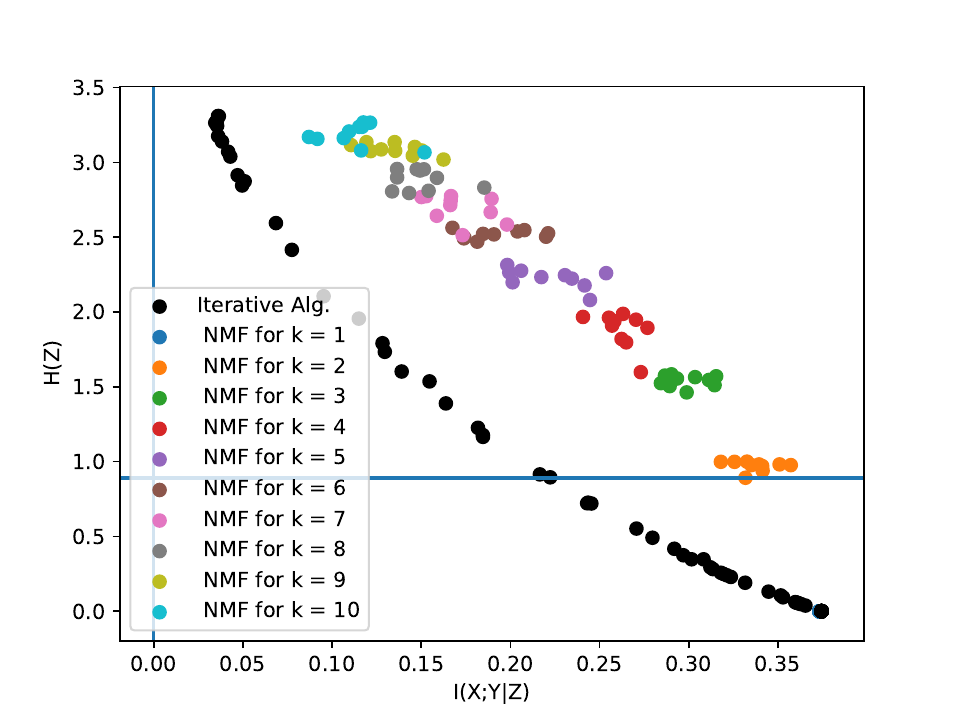}
	\caption{Causal Graph $X\leftarrow Z \rightarrow Y, X\rightarrow Y$}
\end{subfigure}	
\caption{Comparison of the iterative algorithm to NMF for when $|X|=|Y|=20,|Z|=10$. When the true model comes from the causal graph $X\leftarrow Z \rightarrow Y$ in (a), iterative algorithm successfully finds latent variables that with entropy at most true latent entropy (shown as {\color{blue}blue} horizontal line), whereas NMF cannot achieve the same performance, irrespective of the dimension restriction to the latent variable. In (b) data comes from the causal model $X\leftarrow Z \rightarrow Y, X\rightarrow Y$. Although neither algorithm can identify a latent factor that makes $X,Y$ conditionally independent (vertical {\color{blue} blue} line), iterative algorithm finds strictly better latent factors in terms of both small entropy and conditional mutual information between $X,Y$.}
\label{fig:simulation4}
\end{center}
\vskip -0.2in
\end{figure*}